\documentclass{article}
\usepackage[utf8]{inputenc} 
\usepackage[T1]{fontenc}    

\usepackage[a4paper, total={6.25in, 9in}]{geometry}

\usepackage{hyperref}       
\usepackage{url}            
\usepackage{booktabs}       
\usepackage{amsfonts}       
\usepackage{nicefrac}       
\usepackage{microtype}      
\usepackage[round]{natbib}
\usepackage{amsmath,amsthm,amssymb,amsfonts}
\usepackage[dvipsnames,table]{xcolor}         
\usepackage{xspace}
\usepackage{dsfont}
\usepackage{makecell}
\usepackage{nicefrac}
\definecolor{pearThree}{HTML}{E74C3C}
\definecolor{pearcomp}{HTML}{B97E29}
\definecolor{pearDark}{HTML}{2980B9}
\definecolor{pearDarker}{HTML}{1D2DEC}
\hypersetup{
	colorlinks,
	citecolor=pearDark,
	linkcolor=pearThree,
	breaklinks=true,
	urlcolor=black}

\newtheorem{definition}{Definition}
\newtheorem{assumption}{Assumption}
\newtheorem{proposition}{Proposition}
\newtheorem{corollary}{Corollary}
\newtheorem{theorem}{Theorem}
\newtheorem{lemma}{Lemma}

\newcommand{\wb}[1]{\overline{#1}}
\newcommand{\wt}[1]{\widetilde{#1}}
\newcommand{\wh}[1]{\widehat{#1}}

\newcommand{\R}{\mathbb{R}}
\newcommand{\E}{\mathbb{E}}
\renewcommand{\P}{\mathbb{P}}

\newcommand{\prob}[1]{\mathbb{P}\left\{#1\right\}}
\newcommand{\indi}[1]{\mathds{1}\left\{#1\right\}}
\newcommand{\tr}{\mathsf{T}}

\newcommand{\argmax}{\operatornamewithlimits{argmax}}

\renewcommand{\S}{\mathcal{S}} 
\newcommand{\A}{\mathcal{A}} 
\newcommand{\M}{\mathcal{M}} 

\newcommand{\alt}{\Lambda(\M)}
\renewcommand{\O}{\mathcal{O}} 

\newcommand{\weakopt}{\Pi^\star}
\newcommand{\strongopt}{\Pi_{\O}^\star}

\usepackage[colorinlistoftodos, textwidth=25mm]{todonotes}

\usepackage{authblk}

\title{A Fully Problem-Dependent Regret Lower Bound for Finite-Horizon MDPs}

\author[1]{Andrea Tirinzoni}
\author[2]{Matteo Pirotta}
\author[2]{Alessandro Lazaric}
\affil[1]{Inria Lille}
\affil[2]{Facebook AI Research}
\date{}

\usepackage{minitoc}

\begin{document}

\maketitle

\doparttoc 
\faketableofcontents 

\begin{abstract}
We derive a novel asymptotic problem-dependent lower-bound for regret minimization in finite-horizon tabular Markov Decision Processes (MDPs). While, similar to prior work (e.g., for ergodic MDPs), the lower-bound is the solution to an optimization problem, our derivation reveals the need for an additional constraint on the visitation distribution over state-action pairs that explicitly accounts for the dynamics of the MDP. We provide a characterization of our lower-bound through a series of examples illustrating how different MDPs may have significantly different complexity. 1) We first consider a ``difficult'' MDP instance, where the novel constraint based on the dynamics leads to a larger lower-bound (i.e., a larger regret) compared to the classical analysis. 2) We then show that our lower-bound recovers results previously derived for specific MDP instances. 3) Finally, we show that, in certain ``simple'' MDPs, the lower bound is considerably smaller than in the general case and it does not scale with the minimum action gap at all. We show that this last result is attainable (up to $poly(H)$ terms, where $H$ is the horizon) by providing a regret upper-bound based on policy gaps for an optimistic algorithm.
\end{abstract}

\section{Introduction}\label{sec:introduction}
There has been a recent surge of interest for problem-dependent analyses of reinforcement learning (RL) algorithms, both in the context of best policy identification~\citep[e.g.,][]{zanette2019almost,marjani2021adaptive} and regret minimization~\citep[e.g.,][]{simchowitz2019non,he2020logarithmic,Yang2021qlearnin,xu2021fine}.
Before this recent trend, problem-dependent bounds were limited to regret minimization in average-reward Markov decision processes (MDPs)~\citep[e.g.][]{burnetas1997optimal,tewari2007opotimistic,jaksch2010near,ok2018exploration}. Notably, \citet{burnetas1997optimal} derived the first problem-dependent asymptotic lower bound for regret minimization in \emph{ergodic} average-reward MDPs and designed an algorithm matching this fundamental limit. Their lower bound was successively extended by~\citet{ok2018exploration} to structured MDPs. However, these results remain restricted to ergodic MDPs, where the need of exploration is limited to the action space since states are repeatedly visited under any policy.

In finite-horizon MDPs, the literature has focused on deriving problem-dependent ``worst-case'' lower bounds for regret minimization~\citep{simchowitz2019non, xu2021fine} with no state reachability assumption (i.e., the counterpart of ergodicity for finite-horizon MDPs). These results are simultaneously \emph{i)} problem-dependent since they scale with instance-specific quantities (e.g., the action-gaps); \emph{ii)} ``worst-case'' since they are derived only for ``hard'' instances.
Notably, \citet{xu2021fine} proved that there exists a specific MDP such that any consistent algorithm must suffer a regret depending on the inverse of the minimum gap and derived an algorithm with matching regret upper bound.

Despite these results, ``fully'' problem-dependent lower bounds are still missing, i.e., bounds that depend on the properties of any given MDP, instead of relying on specific worst-case instances. In this paper, we take a step in this direction by deriving the first ``fully'' problem-dependent asymptotic regret lower bound for finite-horizon MDPs. Our lower bound generalizes existing results and provides new insights on the ``true'' complexity of exploration in this setting. 
Similarly to average-reward MDPs, our lower-bound is the solution to an optimization problem, but it does not require any assumption on state reachability. Our derivation reveals the need for a constraint on the visitation distribution over state-action pairs that explicitly accounts for the dynamics of the MDP. Interestingly, we show examples where this constraint is crucial to obtain tight lower-bounds and to match existing results derived for specific MDP instances. Finally, we show that, in certain ``simple'' MDPs, the lower bound is considerably smaller than in the general case and it does not scale with the minimum action-gap. We show that this result is attainable (up to $poly(H)$ terms) by providing a novel regret upper-bound for an optimistic algorithm.

\paragraph{Existing asymptotic problem-dependent lower bounds.}

In \emph{ergodic} average-reward MDPs, \citet{burnetas1997optimal,tewari2007opotimistic,ok2018exploration} showed that the optimal problem-dependent regret scales roughly as $O(\sum_{s,a} \frac{\log T}{\Delta(s,a)})$, where $\Delta(s,a)$ is the sub-optimality gap (i.e., the advantage) of action $a$ in state $s$ and $T$ is the number of learning steps.\footnote{More precisely, it scales with a sum of local complexity measures which are related to the sub-optimality gaps~\citep{tewari2007opotimistic}.} This bound has the same shape as the optimal problem-dependent regret in contextual bandits~\citep[e.g.,][]{lattimore2020bandit}. In finite-horizon MDPs, \citet{simchowitz2019non} first showed that the sum of the inverse gaps is a loose lower bound for a specific family of optimistic algorithms, which in the worst-case may suffer from a regret of at least $\Omega(\frac{S}{\Delta_{\min}}\log K)$, where $S$ is the number of states, $K$ is the number of episodes, and $\Delta_{\min}$ is the minimum sub-optimality gap. \citet{xu2021fine} later refined this result showing that there exists a ``hard'' MDP where any consistent algorithm suffers a regret proportional to $\Omega(\frac{SA}{\Delta_{\min}}\log K)$. In such an instance, this bound is exponentially (in $S$) larger than the sum of inverse gaps and it is proportional to $\Omega(\frac{Z_{\mathrm{mul}}}{\Delta_{\min}}\log K)$, where $Z_{\mathrm{mul}}$ is the total number of optimal actions in states where the optimal action is not unique. This suggests that the number of non-unique optimal actions may be key to characterize the ``worst-case'' complexity in finite-horizon MDPs.
\section{Preliminaries}\label{sec:preliminaries}
We consider a time-inhomogeneous finite-horizon MDP $\M := (\S, \A, \{p_h,q_h\}_{h\in[H]}, p_0, H)$~\citep{puterman1994markov}, where $\S$ is a finite set of $S$ states, $\A$ is a finite set of $A$ actions, $p_h : \S\times\A \rightarrow \P(\S)$ and $q_h : \S\times\A \rightarrow \P(\R)$ are the transition probabilities and the reward distribution at stage $h\in [H] := \{1,\ldots,H\}$, $p_0 \in \P(\S)$ is the initial state distribution, and $H$ is the horizon.\footnote{$\P(\Omega)$ denotes the set of probability measures over a set $\Omega$.} We denote by $r_h(s,a)$ the expected reward after taking action $a$ in state $s$. A (deterministic) Markov policy $\pi = \{\pi_h\}_{h\in[H]} \in \Pi$ is a sequence of mappings $\pi_h : \S \rightarrow \A$. Let $\Pi$ be the set of such policies.
Executing a policy $\pi$ on $\M$ yields random trajectories $(s_1, a_1, y_1, \dots, s_H,a_H,y_H)$, where $s_1 \sim p_0$, $a_h = \pi_h(s_h)$, $s_{h+1} \sim p_h(s_h,a_h)$, and $y_{h} \sim q_h(s_h,a_h)$. We denote by $\mathbb{P}_{\M}^\pi,\E_{\M}^\pi$ the corresponding probability and expectation operators, and let $\rho_{\M,h}^\pi(s,a) := \mathbb{P}_\M^\pi\{s_h=s,a_h=a\}$ and $\rho_{\M,h}^\pi(s) := \rho_{\M,h}^\pi(s,\pi_h(s))$ be the state-action and state occupancy measures at stage $h$.
For each $s\in\S$ and $h\in[H]$, we define the action-value function of a policy $\pi$ in $\M$ as
\begin{align*}
    Q_{\mathcal{M},h}^\pi(s,a) := \E_\M^\pi\left[\sum_{h' = h}^H r_{h'}(s_{h'},a_{h'}) | s_h = s, a_h=a\right],
\end{align*}
while the corresponding value function is $V_{\M,h}^\pi(s) := Q_{\mathcal{M},h}^\pi(s,\pi_h(s))$. Let $V_{\M,0}^\pi := \E_{s_1 \sim p_0}[V_{\M,1}^\pi(s_1)]$ be the \emph{expected return} of policy $\pi$ and $V_{\M,0}^\star = \sup_{\pi} V_{\M,0}^\pi$. We define the set of \emph{return-optimal} policies as
\begin{align}\label{eq:return-optimal-policy}
	\weakopt(\M) := \{\pi\in\Pi\ |\ V_{\M,0}^\pi = V_{\M,0}^\star\}.
\end{align}
%
By standard MDP theory~\citep[e.g.,][]{puterman1994markov}, there exists a unique optimal action-value function $Q^\star_{\M,h}$ that satisfies the Bellman optimality equations \emph{for any $h \in [H], s\in\S, a\in\A$},
 \begin{align}\label{eq:bellman-opt}
     Q_{\mathcal{M},h}^\star(s,a) = r_h(s,a) + p_h(s,a)^\tr V_{\M,h+1}^\star,
 \end{align}
where $V_{\M,h}^\star(s) := \max_{a\in\A} Q_{\M,h}^\star(s,a)$.
We define the set of Bellman-optimal actions at state-stage $(s,h)$ as $\mathcal{O}_{\M,h}(s) := \{a\in\A : Q_{\M,h}^\star(s,a) = V_{\M,h}^\star(s)\}$.
Then, the set of \emph{Bellman-optimal} policies is $\strongopt(\M) := \{\pi\in\Pi\ |\ \forall s,h : \pi_h(s) \in \mathcal{O}_{\M,h}(s) \}$. 
A Bellman-optimal policy is always return optimal, i.e., $\strongopt(\M) \subseteq \weakopt(\M)$, while it easy to construct examples where the reverse is not true (i.e., a return-optimal policy is not Bellman optimal). 
Finally, we introduce the \emph{policy gap} $\Gamma_{\M}(\pi) := V_{\M,0}^\star - V_{\M,0}^{\pi}$ and the \emph{action gap} of $a\in\A$ in state $s\in\S$ at stage $h\in[H]$ as
\begin{align}
	\Delta_{\M,h}(s,a) := V^\star_{\M,h}(s) - Q^\star_{\M,h}(s,a).
\end{align}
These two notions of sub-optimality are related by the following equation (proof in App.~\ref{app:auxiliary}):
\begin{align}\label{eq:policy.gap}
	\Gamma_{\M}(\pi) 
	&= \sum_{s\in\S}\sum_{a\in\A} \sum_{h\in[H]}\rho_{\M,h}^\pi(s,a)\Delta_{\M,h}(s,a).
\end{align}
This relationship further shows that a policy $\pi$ can be return-optimal despite selecting actions with $\Delta_{\M,h}(s,a)>0$ (hence it is not Bellman-optimal) at states that have zero occupancy measure $\rho_{\M,h}^\pi(s,a)$.

We consider the standard online learning protocol for finite-horizon MDPs.
At each \emph{episode} $k\in[K]$, the learner plays a policy $\pi_k$ and observes a random trajectory $(s_{k,h}, a_{k,h}, y_{k,h}, \dots, s_{k,H}, a_{k,H}, y_{k,H}) \sim \mathbb{P}_\M^{\pi_k}$. The choice of $\pi_k$ is made by a \emph{learning algorithm} $\mathfrak{A}$, i.e., a measurable function that maps the observations up to episode $k-1$ to policies. The goal is to minimize the cumulative regret,
\begin{align}\label{eq:regret}
    \mathrm{Regret}_K(\M) := \sum_{k=1}^K \Gamma_{\M}(\pi_k) = \sum_{k=1}^K \left(V_{\M,0}^\star - V_{\M,0}^{\pi_k}\right).
\end{align}

\section{Problem-dependent Lower Bound}\label{sec:lower.bound}

As customary in information-theoretic problem-dependent lower bounds, we derive our result for any MDP $\M$ in a given set $\mathfrak{M}$ of MDPs with the same state-action space but different transition probabilities and reward distributions.\footnote{While the focus of this paper is on the standard (unstructured) tabular setting, the set $\mathfrak{M}$ can be used to encode \emph{structure}, i.e., prior knowledge about the problem \citep{ok2018exploration}.}
%
%
Formally, we derive an \emph{asymptotic problem-dependent} lower bound on the expected regret of any ``provably-efficient'' learning algorithm on the set of MDPs $\mathfrak{M}$.
\begin{definition}[$\alpha$-uniformly good algorithm]\label{def:good-alg}
Let $\alpha \in (0,1)$, then a learning algorithm $\mathfrak{A}$ is $\alpha$-uniformly good on $\mathfrak{M}$ if, for each $K\in\mathbb{N}_{>0}$ and $\M\in\mathfrak{M}$, there exists a constant $c(\M)$ such that $\E_{\M}^{\mathfrak{A}}\left[\mathrm{Regret}_K(\M)\right] \leq c(\M)K^\alpha$.
\end{definition}
Note that existing algorithms with $\mathcal{O}(\sqrt{K})$ worst-case regret \citep[e.g.,][]{azar2017minimax,zanette2019tighter} are $\nicefrac{1}{2}$-uniformly good, while those with logarithmic regret \citep[e.g.,][]{simchowitz2019non,xu2021fine} are $\alpha$-uniformly good for all $\alpha\in(0,1)$. 
For the purpose of deriving asymptotic lower bounds, Definition \ref{def:good-alg} can be replaced by assuming that the expected regret is $o(K^\alpha)$ only for sufficiently large $K$ and for any $\alpha\in(0,1)$ \citep{burnetas1997optimal,ok2018exploration}. We make the following assumption on the MDP $\M$ for which we derive the lower bound.
\begin{assumption}[Unique optimal state distribution]\label{ass:unique-opt-rho}
There exists $\rho_{\M,h}^\star \in \P(\S)$ such that, for any optimal policy $\pi \in \Pi^\star(\M)$ and for any $s\in\S,h\in[H]$, $\rho_{\M,h}^\star(s) = \rho_{\M,h}^{\pi}(s)$.
\end{assumption}
This assumption requires all return-optimal policies of $\M$ to induce the same distribution over the state space. This is strictly weaker than assuming a unique optimal action at each state (see Lem.~\ref{lemma:unique-opt-pi-vs-rho} in App.~\ref{app:auxiliary}), as commonly done in the contextual bandit setting \citep{hao2020adaptive,tirinzoni2020asymptotically} and in MDPs \citep{marjani2021adaptive}. 

Let $\mathcal{O}_\M^\star := \{(s,a,h) : s\in\mathrm{supp}(\rho_{\M,h}^\star),a\in\mathcal{O}_{\M,h}(s)\}$ be the set of state-action-stage triplets containing all optimal actions in states that are visited by optimal policies of $\M$. We introduce the following set of \emph{alternative} MDPs to $\M$:
\begin{align*}
         \Lambda(\M) := \Lambda^{\mathrm{wa}}(\M) \cap \Lambda^{\mathrm{wc}}(\M),
    \end{align*}
    where $\Lambda^{\mathrm{wa}}(\M) := \{\M'\in\mathfrak{M} \ |\ \weakopt(\M) \cap \weakopt(\M') = \emptyset \}$ and\footnote{The KL divergence between two MDPs is defined as $\mathrm{KL}_{s,a,h}(\M,\M') = \mathrm{KL}(p_{h}(s,a),p_{h}'(s,a)) + \mathrm{KL}(q_{h}(s,a),q_{h}'(s,a))$.

     }
\begin{align*}
    \Lambda^{\mathrm{wc}}(\M) := \{\M'\in\mathfrak{M} \ |\ \forall  (s,a,h) \in \mathcal{O}_{\M}^\star : \mathrm{KL}_{s,a,h}(\M,\M') = 0 \}.
\end{align*}
The set of alternatives is a key component in the derivation of information-theoretic problem-dependent lower bounds \citep[e.g.,][]{lai1985asymptotically}. 
Similar to \citep{burnetas1997optimal,ok2018exploration}, the set of alternatives $\Lambda(\M)$ is the intersection of two sets: (1) the set of \emph{weak alternatives} $\Lambda^{\mathrm{wa}}(\M)$, i.e., MDPs that have no return-optimal policy in common with $\M$; and (2) the set of \emph{weakly confusing} MDPs $\Lambda^{\mathrm{wc}}(\M)$, i.e., MDPs whose dynamics and rewards are indistinguishable from $\M$ on the state-action pairs observed while executing any return-optimal policy for $\M$.
%
Notice that the set $\Lambda^{\mathrm{wc}}(\M)$ differs from the set of \emph{confusing} MDPs considered in~\citep{burnetas1997optimal,ok2018exploration}. In their case, the zero-KL condition is imposed over all states since the MDP $\M$ is assumed ergodic, which implies that any optimal policy visits all the states with positive probability. In our case, since we do not make any ergodicity assumption, optimal policies may not visit some states at some stages. Therefore, even if the kernels of $\M$ and $\M'$ differ at some optimal action in any such state, the two MDPs remain indistinguishable by playing return-optimal policies. 
With these notions in mind, we are now ready to state our problem-dependent lower bound.

\begin{theorem}\label{th:main-lower-bound}
Let $\mathfrak{A}$ be any $\alpha$-uniformly good learning algorithm on $\mathfrak{M}$ with $\alpha \in (0,1)$. Then, for any $\M\in\mathfrak{M}$ that satisfies Assumption \ref{ass:unique-opt-rho},
\begin{equation*}
\liminf_{K \rightarrow \infty}\frac{\E_{\M}^{\mathfrak{A}}\left[\mathrm{Regret}_K(\M)\right]}{\log(K)} \geq v^\star(\M),
\end{equation*} 
where $v^\star(\M)$ is the value of the optimization problem
%
	
		\begin{equation*}
	\begin{aligned}
    &\underset{\eta\in\mathbb{R}^{SAH}_{\geq 0}}{\inf}&& \sum_{s\in\S}\sum_{a\in\A}\sum_{h\in[H]} \eta_h(s,a)\Delta_{\M,h}(s,a),
     \\
     &
    \quad \mathrm{s.t.} \quad
     && 
        \inf_{\M' \in \alt } \sum_{s\in\S}\sum_{a\in\A} \sum_{h\in[H]}\eta_h(s,a)\mathrm{KL}_{s,a,h}(\M,\M') \geq 1-\alpha,
    \\
    &&&
   \sum_{a\in\A}\eta_h(s,a) = \sum_{s'\in\S}\sum_{a'\in\A}p(s|s',a')\eta_{h-1}(s',a') \quad \forall s\in\S, h>1,
    \\ 
    &&&
    \sum_{a\in\A}\eta_1(s,a) = 0 \quad \forall s\notin\mathrm{supp}(p_0).
	\end{aligned}
	\end{equation*}
\end{theorem}

The lower bound is the solution to a constrained optimization problem that defines an optimal \emph{exploration strategy} $\eta\in\mathbb{R}^{SAH}$, where $\eta_h(s,a)$ is proportional to the number of visits allocated to each state $s$ and action $a$ at stage $h$. Such optimal exploration strategy must minimize the resulting regret (written as a weighted sum of local sub-optimality gaps), while satisfying three constraints. First, the \emph{KL constraint}, which is common in this type of information-theoretic lower bounds, requires that the exploration strategy allocates sufficient visits to relevant state-action-stage triplets so as to discriminate $\M$ from all its alternatives $\M'\in\Lambda(\M)$. The last two constraints, taken as a whole, form what we refer to as the \emph{dynamics constraint}. This requires the optimal exploration strategy to be \emph{realizable} according to (i.e., compatible with) the MDP dynamics. As we shall see in our examples later, the dynamics constraint is a crucial component to introduce MDP structure into the optimization problem. Without it, an exploration strategy would be allowed to allocate visits to certain state-action pairs regardless of the probability to reach them (i.e., as if a generative model were available), thus resulting in a \emph{non-realizable} allocation in most cases and loose lower bounds. 

\subsection{The policy-based perspective}
 
Note that, by definition, we can realize any allocation $\eta$ that satisfies the dynamics constraint by playing some \emph{stochastic} policy. Moreover, we can always express the occupancy measure of any stochastic Markov policy as a mixture of \emph{deterministic} Markov policies~\citep[e.g.,][Remark 6.1, page 64]{altman1999constrained}. 
This implies that an allocation $\eta$ satisfies the dynamics constraint in the optimization problem of Thm.~\ref{th:main-lower-bound} if, and only if, there exists a vector $\omega \in \mathbb{R}^{|\Pi|}_{\geq 0}$ of ``mixing coefficients'' such that $\eta_h(s,a) = \sum_{\pi\in\Pi} \omega_\pi \rho_h^\pi(s,a)$ for all $s,a,h$. 
This allows us to rewrite the optimization problem in a simpler form.
\begin{proposition}\label{prop:policy-based}
The optimization problem of Thm.~\ref{th:main-lower-bound} can be rewritten in the following equivalent form
	\begin{equation*}
	\begin{aligned}
    &\underset{\omega \in \mathbb{R}^{|\Pi|}_{\geq 0}}{\inf}&& \sum_{\pi\in\Pi} \omega_\pi \Gamma_\M(\pi),
     \\
     &
    \quad \mathrm{s.t.} \quad
     && 
        \inf_{\M' \in \alt } \sum_{\pi\in\Pi} \omega_\pi \mathrm{KL}_\pi(\M,\M') \geq 1-\alpha,
	\end{aligned}
	\end{equation*}
	where {\small$\mathrm{KL}_\pi(\M,\M') := \sum_{h\in[H]}\sum_{s\in\S}\sum_{a\in\A} \rho_{\M,h}^\pi(s,a)\mathrm{KL}_{s,a,h}(\M,\M')$}.
\end{proposition}
While computationally harder than its counterpart in Thm.~\ref{th:main-lower-bound} (we moved from optimizing over $SAH$ variables to $|\Pi| = A^{SH}$ variables), this policy-based perspective is convenient to interpret and instantiate the lower bound in specific cases.

\section{Examples}\label{sec:discussion}

We now illustrate a series of examples that show some interesting properties of our lower bound.


\subsection{On the importance of the dynamics constraint to match existing results}
 
We consider the MDP $\M$ introduced by~\citet{xu2021fine} (see Fig.~\ref{fig:example-du-variant} with $\kappa=0$) and we define $\mathfrak{M}$ as the set of MDPs with exactly the same dynamics as $\M$ but arbitrary Gaussian rewards. In this problem $\Delta_{\min} = \varepsilon > 0$. 
We instantiate our lower bound in this case with and without the dynamics constraints.

\begin{corollary}\label{cor:du.example}
	Let $\M$ be the MDP of Fig.~\ref{fig:example-du-variant} with $\kappa=0$. Let $\wt v(\M)$ be the value of the optimization problem of Thm.~\ref{th:main-lower-bound} \textbf{without dynamics constraints}, then 
	$\wt v(\M) = 2(1-\alpha)(\log_2(S+1)+A-2)/\Delta_{\min}$. On the other hand, the lower bound in Thm.~\ref{th:main-lower-bound} \textbf{with dynamics constraints}
	 yields $v^\star(\M) \geq (1-\alpha)SA/\Delta_{\min}$.
\end{corollary}

This result shows that ignoring the dynamics constraints leads to an exponentially smaller (and thus looser) bound w.r.t.\ $v^\star(\M)$. 
On the other hand, when computing the lower bound of Thm.~\ref{th:main-lower-bound}, we match the lower bound of~\citet{xu2021fine} for this configuration.

\subsection{On the usefulness of the policy view}

The policy view of Prop.~\ref{prop:policy-based} is particularly convenient to simplify the expression of the lower bound in some specific cases. One illustrative example is the problem considered above, where all MDPs in $\mathfrak{M}$ share the same dynamics. In this case, the lower bound can be written as
\begin{equation*}
	\begin{aligned}
		&\underset{\omega \in \mathbb{R}^{|\Pi|}_{\geq 0}}{\inf}&& \sum_{\pi\in\Pi} \omega_\pi \langle \theta, \phi^\star - \phi^\pi\rangle,
		\\
		&
		\quad \mathrm{s.t.} \quad
		&& 
		\forall \pi \notin \weakopt(\M) : \|\phi^\pi\|_{D_\omega^{-1}}^2 \leq \frac{\Gamma_{\M}(\pi)^2}{2(1-\alpha)},
	\end{aligned}
\end{equation*}
where $\theta \in \mathbb{R}^{SAH}$ is a vector containing all mean rewards, i.e., $\theta_{s,a,h} = r_h(s,a)$, $\phi^\pi \in \mathbb{R}^{SAH}$ is the vector containing the state distribution $\rho_{\M}^\pi$, $\phi_{s,a,h}^\pi = \rho_{\M,h}^\pi(s,a)$, and $D_\omega := \sum_{\pi\in\Pi} \omega_\pi \mathrm{diag}(\rho_{\M,h}^\pi(s,a)) \in \mathbb{R}^{SAH\times SAH}$ is a diagonal matrix proportional to the number of times each policy is selected. With this notation we have $V_{\M,0}^\pi = \theta^T \phi^\pi$ and $V_{\M,0}^\star = \theta^T \phi^{\star} := \theta^T \phi^{\pi^\star}$ for some optimal policy $\pi^\star$. It is then possible to instantiate this expression for Fig.~\ref{fig:example-du-variant} and obtain the statement of Cor.~\ref{cor:du.example} (see App.~\ref{app:examples}). 

Interestingly, this formulation bears a strong resemblance with the asymptotic lower bound for the \emph{combinatorial semi-bandit} setting \citep[e.g.,][]{wagenmaker2021experimental}. In general, the similarity between the two settings comes from the fact that, in MDPs, there exists a combinatorial number ($A^{SH}$) of policies, which may be treated as arms, whose expected return can be described by only $SAH$ unknown variables (the immediate mean rewards). One difference is that, in combinatorial semi-bandits, the feature vectors usually take values in $\{0,1\}^d$, where $d$ is their dimension ($d=SAH$ in our case). Here, instead, they contain values in $[0,1]^{SAH}$ representing the probabilities that the corresponding policy visits each state-action pair. When the MDP is deterministic, the two problems are indeed equivalent. We remark that the learning feedback itself is the same as the one in combinatorial semi-bandits: whenever we execute a policy $\pi$ in a deterministic MDP, we receive a random reward observation for each state-action-stage triplet associated with an element where $\phi^\pi$ is equal to $1$. When the MDP is not deterministic, on the other hand, we receive the observation only with the corresponding probability. We leave it as a future work to further study the connection and differences between the two settings.

\begin{figure}
    \centering

\tikzset{every picture/.style={line width=0.75pt}} 

\begin{tikzpicture}[x=0.75pt,y=0.75pt,yscale=-1,xscale=0.95]

\draw  [fill={rgb, 255:red, 255; green, 255; blue, 255 }  ,fill opacity=1 ] (122,27.5) .. controls (122,18.94) and (128.94,12) .. (137.5,12) .. controls (146.06,12) and (153,18.94) .. (153,27.5) .. controls (153,36.06) and (146.06,43) .. (137.5,43) .. controls (128.94,43) and (122,36.06) .. (122,27.5) -- cycle ;
\draw   (61,69.5) .. controls (61,60.94) and (67.94,54) .. (76.5,54) .. controls (85.06,54) and (92,60.94) .. (92,69.5) .. controls (92,78.06) and (85.06,85) .. (76.5,85) .. controls (67.94,85) and (61,78.06) .. (61,69.5) -- cycle ;
\draw   (183,69.5) .. controls (183,60.94) and (189.94,54) .. (198.5,54) .. controls (207.06,54) and (214,60.94) .. (214,69.5) .. controls (214,78.06) and (207.06,85) .. (198.5,85) .. controls (189.94,85) and (183,78.06) .. (183,69.5) -- cycle ;
\draw   (26,118.5) .. controls (26,109.94) and (32.94,103) .. (41.5,103) .. controls (50.06,103) and (57,109.94) .. (57,118.5) .. controls (57,127.06) and (50.06,134) .. (41.5,134) .. controls (32.94,134) and (26,127.06) .. (26,118.5) -- cycle ;
\draw   (97,116.5) .. controls (97,107.94) and (103.94,101) .. (112.5,101) .. controls (121.06,101) and (128,107.94) .. (128,116.5) .. controls (128,125.06) and (121.06,132) .. (112.5,132) .. controls (103.94,132) and (97,125.06) .. (97,116.5) -- cycle ;
\draw   (147,118.5) .. controls (147,109.94) and (153.94,103) .. (162.5,103) .. controls (171.06,103) and (178,109.94) .. (178,118.5) .. controls (178,127.06) and (171.06,134) .. (162.5,134) .. controls (153.94,134) and (147,127.06) .. (147,118.5) -- cycle ;
\draw   (218,116.5) .. controls (218,107.94) and (224.94,101) .. (233.5,101) .. controls (242.06,101) and (249,107.94) .. (249,116.5) .. controls (249,125.06) and (242.06,132) .. (233.5,132) .. controls (224.94,132) and (218,125.06) .. (218,116.5) -- cycle ;
\draw [color={rgb, 255:red, 74; green, 144; blue, 226 }  ,draw opacity=1 ]   (123.86,33) -- (88.39,55.4) ;
\draw [shift={(85.86,57)}, rotate = 327.72] [fill={rgb, 255:red, 74; green, 144; blue, 226 }  ,fill opacity=1 ][line width=0.08]  [draw opacity=0] (7.14,-3.43) -- (0,0) -- (7.14,3.43) -- cycle    ;
\draw [color={rgb, 255:red, 74; green, 144; blue, 226 }  ,draw opacity=1 ]   (66.86,81) -- (49.72,102.65) ;
\draw [shift={(47.86,105)}, rotate = 308.37] [fill={rgb, 255:red, 74; green, 144; blue, 226 }  ,fill opacity=1 ][line width=0.08]  [draw opacity=0] (7.14,-3.43) -- (0,0) -- (7.14,3.43) -- cycle    ;
\draw    (86.86,80) -- (103.9,99.73) ;
\draw [shift={(105.86,102)}, rotate = 229.18] [fill={rgb, 255:red, 0; green, 0; blue, 0 }  ][line width=0.08]  [draw opacity=0] (7.14,-3.43) -- (0,0) -- (7.14,3.43) -- cycle    ;
\draw    (186.86,80) -- (169.72,101.65) ;
\draw [shift={(167.86,104)}, rotate = 308.37] [fill={rgb, 255:red, 0; green, 0; blue, 0 }  ][line width=0.08]  [draw opacity=0] (7.14,-3.43) -- (0,0) -- (7.14,3.43) -- cycle    ;
\draw [color={rgb, 255:red, 65; green, 117; blue, 5 }  ,draw opacity=1 ]   (208.86,81) -- (225.9,100.73) ;
\draw [shift={(227.86,103)}, rotate = 229.18] [fill={rgb, 255:red, 65; green, 117; blue, 5 }  ,fill opacity=1 ][line width=0.08]  [draw opacity=0] (7.14,-3.43) -- (0,0) -- (7.14,3.43) -- cycle    ;
\draw [color={rgb, 255:red, 65; green, 117; blue, 5 }  ,draw opacity=1 ]   (152.86,34) -- (185.38,56.3) ;
\draw [shift={(187.86,58)}, rotate = 214.44] [fill={rgb, 255:red, 65; green, 117; blue, 5 }  ,fill opacity=1 ][line width=0.08]  [draw opacity=0] (7.14,-3.43) -- (0,0) -- (7.14,3.43) -- cycle    ;
\draw  [dash pattern={on 0.84pt off 2.51pt}]  (28.86,132) -- (19.86,146) ;
\draw  [dash pattern={on 0.84pt off 2.51pt}]  (35.86,138) -- (34.86,147) ;
\draw  [dash pattern={on 0.84pt off 2.51pt}]  (61.86,145) -- (54.86,132) ;
\draw  [dash pattern={on 0.84pt off 2.51pt}]  (104.86,132) -- (95.86,146) ;
\draw  [dash pattern={on 0.84pt off 2.51pt}]  (120.86,132) -- (129.86,145) ;
\draw  [dash pattern={on 0.84pt off 2.51pt}]  (154.86,134) -- (145.86,148) ;
\draw  [dash pattern={on 0.84pt off 2.51pt}]  (170.86,134) -- (179.86,147) ;
\draw  [dash pattern={on 0.84pt off 2.51pt}]  (226.86,134) -- (217.86,148) ;
\draw  [dash pattern={on 0.84pt off 2.51pt}]  (242.86,134) -- (251.86,147) ;

\draw (131,18) node [anchor=north west][inner sep=0.75pt]  [font=\footnotesize] [align=left] {$\displaystyle s_{1}^{1}$};
\draw (69,61) node [anchor=north west][inner sep=0.75pt]  [font=\footnotesize] [align=left] {$\displaystyle s_{1}^{2}$};
\draw (192,61) node [anchor=north west][inner sep=0.75pt]  [font=\footnotesize] [align=left] {$\displaystyle s_{2}^{2}$};
\draw (33,109) node [anchor=north west][inner sep=0.75pt]  [font=\footnotesize] [align=left] {$\displaystyle s_{1}^{H}$};
\draw (104,107) node [anchor=north west][inner sep=0.75pt]  [font=\footnotesize] [align=left] {$\displaystyle s_{2}^{H}$};
\draw (155,109) node [anchor=north west][inner sep=0.75pt]  [font=\footnotesize] [align=left] {$\displaystyle s_{3}^{H}$};
\draw (226,107) node [anchor=north west][inner sep=0.75pt]  [font=\footnotesize] [align=left] {$\displaystyle s_{n_{{\text{\tiny H}}}}^{H}$};
\draw (26,144) node [anchor=north west][inner sep=0.75pt]  [font=\footnotesize] [align=left] {$\displaystyle a_{2} \dotsc a_{m}$};
\draw (90,144) node [anchor=north west][inner sep=0.75pt]  [font=\footnotesize] [align=left] {$\displaystyle a_{1} \dotsc a_{m}$};
\draw (141,144) node [anchor=north west][inner sep=0.75pt]  [font=\footnotesize] [align=left] {$\displaystyle a_{1} \dotsc a_{m}$};
\draw (210,144) node [anchor=north west][inner sep=0.75pt]  [font=\footnotesize] [align=left] {$\displaystyle a_{1} \dotsc $};
\draw (13,161) node [anchor=north west][inner sep=0.75pt]  [font=\footnotesize,color={rgb, 255:red, 74; green, 144; blue, 226 }  ,opacity=1 ] [align=left] {$\displaystyle \varepsilon $};
\draw (28,159) node [anchor=north west][inner sep=0.75pt]  [font=\footnotesize,color={rgb, 255:red, 208; green, 2; blue, 27 }  ,opacity=1 ] [align=left] {$\displaystyle 0$};
\draw (60,159) node [anchor=north west][inner sep=0.75pt]  [font=\footnotesize,color={rgb, 255:red, 208; green, 2; blue, 27 }  ,opacity=1 ] [align=left] {$\displaystyle 0$};
\draw (92,159) node [anchor=north west][inner sep=0.75pt]  [font=\footnotesize,color={rgb, 255:red, 208; green, 2; blue, 27 }  ,opacity=1 ] [align=left] {$\displaystyle 0$};
\draw (124,159) node [anchor=north west][inner sep=0.75pt]  [font=\footnotesize,color={rgb, 255:red, 208; green, 2; blue, 27 }  ,opacity=1 ] [align=left] {$\displaystyle 0$};
\draw (143,159) node [anchor=north west][inner sep=0.75pt]  [font=\footnotesize,color={rgb, 255:red, 208; green, 2; blue, 27 }  ,opacity=1 ] [align=left] {$\displaystyle 0$};
\draw (175,159) node [anchor=north west][inner sep=0.75pt]  [font=\footnotesize,color={rgb, 255:red, 208; green, 2; blue, 27 }  ,opacity=1 ] [align=left] {$\displaystyle 0$};
\draw (211,159) node [anchor=north west][inner sep=0.75pt]  [font=\footnotesize,color={rgb, 255:red, 208; green, 2; blue, 27 }  ,opacity=1 ] [align=left] {$\displaystyle 0$};
\draw (12,145) node [anchor=north west][inner sep=0.75pt]  [font=\footnotesize,color={rgb, 255:red, 74; green, 144; blue, 226 }  ,opacity=1 ] [align=left] {$\displaystyle a_{1}$};
\draw (98,27) node [anchor=north west][inner sep=0.75pt]  [font=\footnotesize,color={rgb, 255:red, 74; green, 144; blue, 226 }  ,opacity=1 ] [align=left] {L};
\draw (46,75) node [anchor=north west][inner sep=0.75pt]  [font=\footnotesize,color={rgb, 255:red, 74; green, 144; blue, 226 }  ,opacity=1 ] [align=left] {L};
\draw (168,74) node [anchor=north west][inner sep=0.75pt]  [font=\footnotesize] [align=left] {L};
\draw (168,29) node [anchor=north west][inner sep=0.75pt]  [font=\footnotesize,color={rgb, 255:red, 65; green, 117; blue, 5 }  ,opacity=1 ] [align=left] {R};
\draw (221,76) node [anchor=north west][inner sep=0.75pt]  [font=\footnotesize,color={rgb, 255:red, 65; green, 117; blue, 5 }  ,opacity=1 ] [align=left] {R};
\draw (100,75) node [anchor=north west][inner sep=0.75pt]  [font=\footnotesize] [align=left] {R};
\draw (241,144) node [anchor=north west][inner sep=0.75pt]  [font=\footnotesize,color={rgb, 255:red, 65; green, 117; blue, 5 }  ,opacity=1 ] [align=left] {$\displaystyle a_{m}$};
\draw (245,161) node [anchor=north west][inner sep=0.75pt]  [font=\footnotesize,color={rgb, 255:red, 65; green, 117; blue, 5 }  ,opacity=1 ] [align=left] {$\displaystyle \kappa$};

\end{tikzpicture}
    
\vspace{-0.1in}
    \caption{Variant of the example in~\citep{xu2021fine}. The MDP is binary tree with $S=2^H-1$ states, $A=m\geq 2$ actions, and deterministic transitions. The figure shows an instance with $H=3$. The agent starts from the root state $s_1^1$ and descends the tree using only two actions ($L$ and $R$). In the leaf states, all the $m$ actions are available. The rewards follow a Gaussian distribution with unit variance and mean equal to zero everywhere except for at most two leaf state-action pairs (whose values are $\varepsilon$ and $\kappa$).}
    \label{fig:example-du-variant}
    \vspace{-0.1in}
\end{figure}
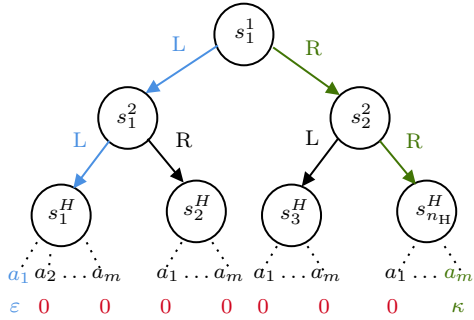


\subsection{On the dependence on the sum of inverse gaps} 

While the lower bound of~\citet{xu2021fine} shows that there exists an MDP where the regret is significantly larger than the sum of inverse gaps whenever multiple equivalent optimal actions exist, in the following we derive a result that is somewhat complementary: we show that there exists a large class of MDPs where the lower bound scales as the sum of the inverse gaps, even when $Z_{\mathrm{mul}} > 0$.

\begin{proposition}\label{prop:full-support}
Let $\M$ be an MDP satisfying Asm.~\ref{ass:unique-opt-rho} such that $\rho_{\M}^\star$ is \emph{full-support} 
(i.e., $\rho_{\M,h}^\star(s)>0$, $\forall s,h$). Then, 
\begin{align*}
v^\star(\M) = (1-\alpha)\sum_{h,s,a}\frac{\Delta_{\M,h}(s,a)}{\mathcal{K}_{s,a,h}(\M)} \leq \sum_{h,s,a}\frac{2(H-h)^2}{\Delta_{\M,h}(s,a)},
\end{align*}
where $\mathcal{K}_{s,a,h}(\M) : = \inf_{\bar{p},\bar{q} \in \Lambda_{s}(\M)}\big\{ \mathrm{KL}(p_h(s,a), \bar{p}) + \mathrm{KL}(q_h(s,a), \bar{q})\big\}$ and $\Lambda_{s}(\M) := \{\bar{p}\in\P(\S),\bar{q}\in\P([0,1]) : \mathbb{E}_{y\sim \bar{q}}[y] + \bar{p}^T V_{\M,h+1}^\star > V_{\M,h}^\star(s)\}$. 
\end{proposition}
Note that the full-support condition for $\rho_{\M,h}^\star$ is weaker than ergodicity for average-reward MDPs since it is required only for the optimal policy. For MDPs with this property, the lower bound is obtained by a \emph{decoupled} exploration strategy similar to the one for ergodic MDPs, where the optimal allocation focuses on exploring sub-optimal actions regardless of how to reach the corresponding state, while the exploration of the state space comes ``for free'' from trying to minimize regret w.r.t.\ the optimal policy itself. Interestingly, this result holds even for $Z_{\mathrm{mul}} > 0$, suggesting that the dependency $\frac{Z_{\mathrm{mul}}}{\Delta_{\min}}$ derived in~\citep{xu2021fine} may be relevant only under specific reachability properties (e.g., when the optimal occupancy measure is not full support).


\subsection{On the dependence on the minimum gap}\label{sec:example-min-gap}

Let us consider again the MDP of Fig.~\ref{fig:example-du-variant} under the same setting as before except that $\kappa \geq 2\varepsilon > 0$. In this problem, $\Delta_{\min} = \epsilon$ and $\Delta_{\max} = \kappa$ are the minimum and maximum action gap, respectively. Perhaps surprisingly, despite we only added a single positive reward ($\kappa$) to the original hard instance of \citet{xu2021fine}, we now show that the lower bound of Thm.~\ref{th:main-lower-bound} does not scale with the minimum gap at all.
\begin{proposition}\label{prop:delta.max}
Let $\M$ be the MDP of Fig.~\ref{fig:example-du-variant} with $\kappa \geq 2\varepsilon > 0$, then the lower bound of Thm.~\ref{th:main-lower-bound} yields $v^\star(\M) \leq 12(1-\alpha)\frac{SA}{\Delta_{\max}}$. On the other hand, the sum of inverse gaps of $\M$ is at least $(\log_2(S+1)+A-3)/\Delta_{\min}$.
\end{proposition}
This result shows that, for given $S,A,H$, one can always construct an MDP where the lower bound of Thm.~\ref{th:main-lower-bound} is smaller than the sum of inverse gaps by an arbitrarily large factor. The intuition is as follows. In order to learn an optimal policy, any consistent algorithm must figure out which among actions $L$ and $R$ is optimal at the root state $s_1^1$. Action $L$ leads to a return of $\epsilon$, while action $R$ yields a (possibly much larger) return $\kappa$. Suppose the agent has estimated all the rewards in the MDP up to an error of $\kappa/2$. This is enough for it to ``prune'' the whole left branch of the tree since its return is certainly smaller than the one in the right branch. This is better illustrated using the \emph{policy view}: each policy in this MDP has a gap $\Gamma(\pi)\geq\kappa/2$. Thus, an estimation error below the minimum policy gap suffices to discriminate all sub-optimal policies w.r.t.\ the optimal one. Notably, this means that the left branch need not be explored to gain $\epsilon$-accurate estimates, which would translate into a much larger $O(1/\Delta_{\min})$ regret. In other words, the agent is not required to explore until it learns a \emph{Bellman optimal} policy (i.e., one that correctly chooses action $a_1$ in state $s_1^H$); any \emph{return optimal} policy suffices to minimize regret, and this can be obtained by simply learning to take the right path while playing arbitrary actions at all other states.


\section{Policy-Gap-Based Regret Bound for Optimistic Algorithms}

To confirm that the result of Sec.~\ref{sec:example-min-gap} is not an artifact of our lower bound, we provide a novel (logarithmic) problem-dependent regret bound for the optimistic algorithm UCBVI~\citep{azar2017minimax} that scales with the minimum \emph{policy gap} $\Gamma_{\min}$ instead of the action gaps $\Delta_{\M}$ as in recent works~\citep{simchowitz2019non,xu2021fine}. We first recall the general template of optimistic algorithms \citep[e.g.,][]{simchowitz2019non}. We define the optimistic action-value function as
\begin{align*}
    \wb{Q}_h^k(s,a) := \wh{r}_h^k(s,a) + \wh{p}_h^k(s,a)^T V_{h+1}^k + c_h^k(s,a),
\end{align*}
where $\wh{r}_h^k, \wh{p}_h^k$ are the empirical mean rewards and transition probabilities, respectively, while $c_h^k(s,a)$ is a confidence bonus to ensure optimism. The corresponding optimistic value function is $\wb{V}_h^k(s) := \max_a \wb{Q}_h^k(s,a)$. At each episode $k\in[K]$, the agent plays the policy $\pi_h^k(s) := \argmax_a  \wb{Q}_h^k(s,a)$. 

We carry out the analysis for an MDP $\M$ with unknown rewards and transition probabilities and, for simplicity, deterministic initial state $s_1$. As common, we only assume that rewards lie in $[0,1]$ almost surely.\footnote{The proof follows equivalently by assuming sub-Gaussian rewards with sub-Gaussian factor $\sigma^2 = 1$ and mean in $[0,1]$.} Since our goal is to show the dependence on the policy gaps, we shall choose the simple Chernoff-Hoeffding-based confidence bonus \citep{azar2017minimax}, which yields a regret bound that is sub-optimal in the (minimax) dependence on $H$. The analysis can be conducted analogously (and improved in the dependence on $H$) with Bernstein-based bonuses following recent works \citep{zanette2019tighter,simchowitz2019non}.

\begin{theorem}\label{lem:ucbvi.upper}
	Let $\M$ be any MDP with rewards in $[0,1]$ and $K \geq 1$, then the expected regret of UCBVI with Chernoff-Hoeffding bonus (ignoring low-order terms in $\log(K))$ is
	\begin{align*}
	\mathbb{E}_{\M}[\mathrm{Regret}(K)] \lesssim \frac{4H^4 SA}{\Gamma_{\min}}\log (4SAHK^2).
	\end{align*}
\end{theorem}
This result shows that \emph{1)} UCBVI attains the result in Prop.~\ref{prop:delta.max} (up to $\text{poly}(H)$ factors) where $\Gamma_{\min} = \Delta_{\max}$, even when dynamics are unknown; \emph{2)} Prop.~\ref{prop:delta.max} is tight w.r.t.\ the gaps; \emph{3)} it is possible to achieve regret not scaling with $\Delta_{\min}$.

\section{Discussion} 

While Thm.~\ref{th:main-lower-bound} provides the first ``fully'' problem-dependent lower bound for finite-horizon MDPs, it opens a number of interesting directions. \textbf{1)} As all existing problem-dependent lower bounds for this setting, the result is asymptotic in nature. A more refined finite-time analysis could be obtained following~\citet{garivier2019explore}. 
\textbf{2)} 
For the case studied in Cor.~\ref{cor:du.example}, \citet{xu2021fine} provide an algorithm with matching upper bound (up to $\text{poly}(H)$ factors), while we provide a matching upper bound (Thm.~\ref{lem:ucbvi.upper}) for Prop.~\ref{prop:delta.max}. It remains an open question how to design an algorithm to match the bound of Thm.~\ref{th:main-lower-bound}.
\textbf{3)} Most of the ingredients in deriving and analyzing Thm.~\ref{th:main-lower-bound} could be adapted to the average-reward case to obtain a lower bound with no ergodicity assumption. 

\bibliography{bibliography}
\bibliographystyle{plainnat}

\appendix

\part{Appendix}


\parttoc
\newpage

\section{Problem-Dependent Lower Bound}\label{app:lower.bound}

In this section, we prove our main lower bound (Theorem~\ref{th:main-lower-bound}), analyze its properties, and instatiate it in some interesting cases. We recall that we consider an arbitrary set of MDPs $\mathfrak{M}$ and an MDP instance $\M\in\mathfrak{M}$ that satisfies Assumption~\ref{ass:unique-opt-rho}.

\subsection{The Set of Alternative MDPs}\label{app:alternative}

Throughout the appendix, we shall rewrite the set of alternatives to $\M$ introduced in the main paper in the following equivalent form:
\begin{align}\label{eq:alt-set}
    \Lambda(\M) := \{\M'\in\mathfrak{M} \ |\ \strongopt(\M) \cap \weakopt(\M') = \emptyset, \forall  s,a,h \in \mathcal{O}_{\M}^\star : \mathrm{KL}_{s,a,h}(\M,\M') = 0 \}.
\end{align}
Note that we replaced the set of \emph{return-optimal} policies for $\M$ (and only for $\M$) with the set of its \emph{Bellman-optimal} policies. The proof that this set actually coincides with the one in the main paper can be found in Lemma~\ref{lemma:equiv-weak-alt} in Appendix~\ref{app:auxiliary}.

For this set, we prove the following properties, which shall be crucial for deriving the lower bound.
\begin{lemma}\label{lemma:alt-properties}
Let $\M$ satisfy Assumption \ref{ass:unique-opt-rho} and $\M'$ be an MDP such that (1) {$\strongopt(\M) \cap \weakopt(\M') = \emptyset$} (equivalently, $V_{\M',0}^\star > V_{\M',0}^\pi$ for all $\pi\in\strongopt(\M)$), and (2) $\mathrm{KL}_{s,a,h}(\M,\M') = 0$ for all $h\in[H], s\in\mathrm{supp}(\rho_{\M,h}^\star),a\in\mathcal{O}_{\M,h}(s)$. Then, all of the following properties hold:
\begin{enumerate}
    \item For each optimal policy $\pi\in\Pi^\star(\M)$: $\rho_\M^\star = \rho_\M^{\pi} = \rho_{\M'}^{\pi}$ and $V_{\M,0}^\star = V_{\M,0}^\pi = V_{\M',0}^\pi$; 
    \item There exist $h\in[H], s\in\mathrm{supp}(\rho_{\M,h}^\star)$ such that, for all actions $a\in\mathcal{O}_{\M,h}(s)$, $Q_{\M',h}^\star(s,a) < V_{\M',h}^\star(s)$.
\end{enumerate}
\end{lemma}
\begin{proof}
Let us start by proving the first property. Take any optimal policy $\pi\in\Pi^\star(\M)$ for $\M$. By assumption, $\pi$ has the same state distribution $\rho_\M^\star$ as all other optimal policies of $\M$. Moreover, Lemma \ref{lemma:opt-act-vs-rho} ensures that, for each state $s\in\mathrm{supp}(\rho_{\M,h}^\star)$, the action prescribed by $\pi$ must be in $\mathcal{O}_{\M,h}(s)$. Therefore, by the second condition on $\M'$, we have that the kernels of the Markov reward processes induced by $\pi$ on $\M$ and $\M'$ are exactly the same on states $s\in\mathrm{supp}(\rho_{\M,h}^\star)$. Since the kernels at any state $s\notin\mathrm{supp}(\rho_{\M,h}^\star)$ do not influence the state distribution of $\pi$, we must have $\rho_{\M}^\star = \rho_{\M'}^\pi$. Moreover,
\begin{align*}
    V_{\M',0}^{\pi} &= \sum_{s\in\S}\sum_{h\in[H]}\rho_{\M',h}^\pi(s)r_{h}'(s,\pi_{h}(s)) 
    \\ &= \sum_{h\in[H]}\sum_{s\in\mathrm{supp}(\rho_{\M,h}^\star)}\rho_{\M',h}^\pi(s)\underbrace{r_{h}'(s,\pi_{h}(s))}_{= r_{h}(s,\pi_{h}(s))} + \sum_{h\in[H]}\sum_{s\notin\mathrm{supp}(\rho_{\M,h}^\star)}\underbrace{\rho_{\M',h}^\pi(s)}_{=0}r_{h}'(s,\pi_{h}(s)) 
    \\ &= \sum_{h\in[H]}\sum_{s\in\mathrm{supp}(\rho_{\M,h}^\star)}\rho_{\M,h}^\pi(s)r_{h}(s,\pi_{h}(s))
    = V_{\M,0}^\pi = V_{\M,0}^\star.
\end{align*}
This proves the first property.

We prove the second property by contradiction. Suppose that, for all $h\in[H], s\in\mathrm{supp}(\rho_{\M,h}^\star)$, there exists an action $a\in\mathcal{O}_{\M,h}(s)$ such that $Q_{\M',h}^\star(s,a) = V_{\M',h}^\star(s)$. Take a policy $\pi\in\strongopt(\M)$ that chooses these actions at all states $s\in\mathrm{supp}(\rho_{\M,h}^\star)$. Moreover, build another policy $\pi'$ which is equal to $\pi$ in all states $s\in\mathrm{supp}(\rho_{\M,h}^\star)$ and that chooses $\argmax_{a\in\A} Q_{\M',h}^\star(s,a)$ at all other states. Clearly, $\pi'$ satisfies the Bellman optimality equations at all states and stages, hence it is optimal for $\M'$. Moreover, $\pi$ and $\pi'$ have the same state distribution and the same expected return. Thus, by the first property derived above,
\begin{align*}
    V_{\M',0}^\star = V_{\M',0}^{\pi'} = V_{\M',0}^{\pi},
\end{align*}
which contradicts the assumption that all (strongly) optimal policies of $\M$ are sub-optimal in $\M'$. This concludes the proof.
\end{proof}

\subsection{Behavior of Uniformly-Good Algorithms}\label{app:behavior-good}

Characterizing the behavior of the class of algorithms under consider is a key aspect for deriving informatiion-theoretic problem-dependent lower bounds. Here we prove some important properties of $\alpha$-uniformly good algorithms. The following one is almost immediate from Definition \ref{def:good-alg}.

\begin{lemma}[Behavior of good algorithms]\label{lemma:bound-visits-consistent}
Let $\mathfrak{A}$ be an $\alpha$-uniformly good algorithm on $\mathfrak{M}$. Then, for any $\M\in\mathfrak{M}$, the following two inequalities hold:
\begin{align}\label{eq:ub-visit-consistent}
    \forall s\in\S,h\in[H],a\notin\mathcal{O}_{\M,h}(s) : \E_\M^{\mathfrak{A}}[N_{K,h}(s,a)] \leq \frac{{c}(\M)}{\Delta_{\M}}K^\alpha,
\end{align}
\begin{align}\label{eq:ub-policies,consistent}
     \E_\M^{\mathfrak{A}}[M_K] \leq \frac{{c}(\M)}{\Gamma_{\M}}K^\alpha,
\end{align}
where $\Delta_{\M} := \min_{s,a,h : \Delta_{\M,h}(s,a) > 0} \Delta_{\M,h}(s,a)$ and $\Gamma_{\M} := V_{\M,0}^\star - \max_{\pi\notin \Pi^\star(\M) }V_{\M,0}^\pi$ are the minimum action and policy gap, respectively.
\end{lemma}
\begin{proof}
Using Proposition \ref{prop:regret-decomp} together with the definition of minimum gap and $\alpha$-uniformly good algorithm, for any $ s\in\S,h\in[H]$ and sub-optimal action $a\notin\mathcal{O}_{\M,h}(s)$,
\begin{align*}
    c(\M)K^\alpha \geq \E_{\M}^{\mathfrak{A}}\left[\mathrm{Regret}_K( \M)\right] \geq \E_\M^\mathfrak{A}[N_{K,h}(s,a)]\Delta_{\M,h}(s,a) \geq \E_\M^\mathfrak{A}[N_{K,h}(s,a)]\Delta_{\M}.
\end{align*}
Rearranging proves \eqref{eq:ub-visit-consistent}. Moreover, using the definition of regret and policy gap,
\begin{align*}
        c(\M)K^\alpha \geq \E_{\M}^{\mathfrak{A}}\left[\mathrm{Regret}_K(\M)\right] := \E_{\M}^{\mathfrak{A}}\left[\sum_{k=1}^K \left(V_{\M,0}^\star - V_{\M,0}^{\pi_k} \right)\right] = \sum_{k=1}^K \Gamma_{\M}(\pi_k) \geq \Gamma_\M \E_{\M}^{\mathfrak{A}}\left[M_K\right],
\end{align*}
which proves \eqref{eq:ub-policies,consistent}.
\end{proof}

The following is the most important result behind the derivation of a lower bound that does not use any reachability assumption (e.g., ergodicity). It shows that, for MDPs where the optimal state distribution is unique, the empirical state-visitation frequencies of any $\alpha$-uniformly good algorithm eventually converge to the optimal ones. As we shall see in the lower bound derivation, this allows to characterize precisely what states are visited with high frequency by good algorithms, thus avoiding the need to impose algorithm-agnostic conditions on the state visitations (e.g., ergodicity).

\begin{lemma}[State distribution of good algorithms]\label{lemma:state-dist-consistent-prob}
Let $\mathfrak{A}$ be an $\alpha$-uniformly good algorithm on $\mathfrak{M}$. Then, for any MDP $\M\in\mathfrak{M}$ with unique optimal state distribution $\rho_{\M}^\star$, for any $s\in\S$ and $h\in[H]$,
\begin{align*}
    \P_\M^{\mathfrak{A}}\left\{ \left| N_{K,h}(s) - K\rho_{\M,h}^{\star}(s)\right| > \frac{4{c}(\M)K^{\alpha}}{\Gamma_{\M} \delta} + \sqrt{\frac{K}{2}\log\frac{4}{\delta}} \right\} \leq \delta.
\end{align*}
\end{lemma}
\begin{proof}
Using the triangle inequality,
\begin{align*}
        &\left| \frac{N_{K,h}(s)}{K} - \rho_{\M,h}^{\star}(s)\right|
        = \frac{1}{K}\left| \sum_{k=1}^K \indi{s_{k,h}=s, \pi_k\in\Pi^\star(\M)} + \sum_{k=1}^K \indi{s_{k,h}=s, \pi_k\notin\Pi^\star(\M)} - K\rho_{\M,h}^{\star}(s)\right|
        \\ & \quad \leq \frac{M_K}{K} + \frac{1}{K}\left| \sum_{k=1}^K \indi{s_{k,h}=s, \pi_k\in\Pi^\star(\M)} - K\rho_{\M,h}^{\star}(s)\right|
        \\ & \quad = \frac{M_K}{K} + \frac{1}{K}\left| \sum_{k=1}^K \indi{s_{k,h}=s, \pi_k\in\Pi^\star(\M)} \pm  \sum_{k=1}^K \E_\M^\mathfrak{A}\left[\indi{s_{k,h}=s, \pi_k\in\Pi^\star(\M)} \big| \mathcal{F}_{k-1} \right] - K\rho_{\M,h}^{\star}(s)\right|.
\end{align*}
Note that, since $\pi_k$ is $\mathcal{F}_{k-1}$-measurable, we have
\begin{align*}
    \E_\M^\mathfrak{A}\left[\indi{s_{k,h}=s, \pi_k\in\Pi^\star(\M)} \big| \mathcal{F}_{k-1} \right] &=
    \indi{\pi_k\in\Pi^\star(\M)}\E_\M^\mathfrak{A}\left[\indi{s_{k,h}=s, }\big| \mathcal{F}_{k-1} \right] 
    \\ &= \indi{\pi_k\in\Pi^\star(\M)} \rho_{\M,h}^\star(s),
\end{align*}
where the last equality is due to the fact that either the expectation of the first line is equal to $\rho_{\M,h}^\star(s)$ (when $\pi_k$ is optimal) or the whole term is zero (when $\pi_k$ is sub-optimal). Therefore, 
\[
    \sum_{k=1}^K \E_\M^\mathfrak{A}\left[\indi{s_{k,h}=s, \pi_k\in\Pi^\star(\M)} \big| \mathcal{F}_{k-1} \right] = (K - M_K) \rho_{\M,h}^\star(s).
\] 
Let $E_k := \{s_{k,h}=s, \pi_k\in\Pi^\star(\M)\}$. Plugging this back into the previous decomposition, and after another application of the triangle inequality,
\begin{align*}
        &\left| \frac{N_{K,h}(s)}{K} - \rho_{\M,h}^{\star}(s)\right|
        \leq \frac{2 M_K}{K} + \frac{1}{K}\left| \sum_{k=1}^K \left(\indi{E_k} - \E_\M^\mathfrak{A}\left[\indi{E_k} \big| \mathcal{F}_{k-1} \right]\right)\right|.
\end{align*}
Let $\delta'\in(0,1)$. Since $E_k$ is $\mathcal{F}_{k-1}$-measurable, the second term above is a martingale with differences bounded by $1$ in absolute value. Therefore, using Azuma's inequality
\begin{align*}
    \P_\M^{\mathfrak{A}}\left\{ \left| \sum_{k=1}^K \left(\indi{E_k} - \E_\M^\mathfrak{A}\left[\indi{E_k} \big| \mathcal{F}_{k-1} \right]\right)\right| > \sqrt{\frac{K}{2}\log\frac{2}{\delta'}}\right\} \leq \delta'.
\end{align*}
Moreover, an application of Markov's inequality followed by Lemma \ref{lemma:bound-visits-consistent} yields
\begin{align*}
    \P_\M^{\mathfrak{A}}\left\{ M_K > \frac{{c}(\M)}{\Gamma_{\M}}\frac{K^\alpha}{\delta'}\right\} \leq \delta' \frac{\E_\M^{\mathfrak{A}}[M_K] }{{c}(\M)K^\alpha/\Gamma_{\M}} \leq \delta'.
\end{align*}
Setting $\delta' = \delta/2$ and taking a union bound to cover both inequalities concludes the proof. 
\end{proof}

\subsection{Lower Bound Derivation (Proof of Theorem~\ref{th:main-lower-bound})}

Now that we illustrated the properties of the set of alternatives and analyzed the behavior of good algorithms, we are ready to derive our main lower bound. The properties of the set of alternatives $\alt$ (Appendix \ref{app:alternative}) will allow us to use standard change-of-measure arguments, while the convergence in state-distribution of $\alpha$-uniformly good algorithms (Appendix \ref{app:behavior-good}) will be crucial to avoid any reachability assumption. The following is the key result towards deriving the lower bound.

\begin{theorem}\label{th:constr-lower-bound}
Let $\mathfrak{A}$ be any $\alpha$-uniformly good learning algorithm on $\mathfrak{M}$ with $\alpha \in (0,1)$. Then, for any MDP $\M \in \mathfrak{M}$ that satisfies Assumption \ref{ass:unique-opt-rho},
\begin{equation}
\inf_{\M' \in \Lambda(\M)}\liminf_{K\rightarrow\infty}\sum_{s\in\S}\sum_{a\in\A}\sum_{h\in[H]} \frac{\E_\M^{\mathfrak{A}}[N_{K,h}(s,a)]}{\log(K)}\mathrm{KL}_{s,a,h}(\M,\M')  \geq 1 - \alpha,
\end{equation} 
where $\alt$ is the set defined in \eqref{eq:alt-set}.
\end{theorem}
\begin{proof}
For any MDP model $\M'$ and $\mathcal{F}_K$-measurable event $\mathcal{E}$, Lemma \ref{lemma:change-of-measure} ensures that
\begin{align}\label{eq:change-of-mea-lb}
     \sum_{s\in\S}\sum_{a\in\A}\sum_{h\in[H]} \E_\M^{\mathfrak{A}}[N_{K,h}(s,a)]\mathrm{KL}_{s,a,h}(\M,\M') \geq \mathrm{kl}(\mathbb{P}_{\M}^{\mathfrak{A}}\{\mathcal{E}\}, \mathbb{P}_{\M'}^{\mathfrak{A}}\{\mathcal{E}\}) \geq \mathbb{P}_{\M}^{\mathfrak{A}}\{\mathcal{E}\}\log \frac{1}{\mathbb{P}_{\M'}^{\mathfrak{A}}\{\mathcal{E}\}} - \log 2,
\end{align}
where the last inequality is easy to check from the definition of Bernoulli KL divergence (see, e.g., Lemma 15 of \cite{domingues2021episodic}). We now show that, for any alternative MDP model $\M'\in\alt$, we can find a suitable event $\mathcal{E}$ such that the right-hand side grows at a sufficient rate. 

Fix any alternative MDP $\M' \in \alt$. By the second property of Lemma \ref{lemma:alt-properties}, there exist $h\in[H], s\in\mathrm{supp}(\rho_{\M,h}^\star)$ such that, for all actions $a\in\mathcal{O}_{\M,h}(s)$, $\Delta_{\M',h}(s,a) > 0$. Take one such couple $s,h$ and consider $\O_{\M',h}(s) := \{a\in\A : Q_{\M',h}^\star(s,a) = V_{\M',h}^\star(s)\}$, i.e, the set of actions that are optimal in state $s$ at stage $h$ for MDP $\M'$. Note that, by Lemma \ref{lemma:opt-act-vs-rho}, $\pi_h(s) \notin \O_{\M',h}(s)$ for all optimal policies $\pi\in\Pi^\star(\M)$ for $\M$. Define the event
\begin{align}
\mathcal{E} := \{N_{K,h}(s) \geq  \rho_{\M,h}^\star(s)K - f_{\alpha}(K), \sum_{a\in\O_{\M',h}(s)}N_{K,h}(s,a) < K^{\alpha + \beta}\},
\end{align}
where
\begin{align*}
    f_\alpha(K) := \frac{4{c}(\M)K^{\alpha+\xi}}{\Gamma_{\M}} + \sqrt{\frac{K}{2}\log(4K^\xi)}.
\end{align*}
and $\beta,\xi > 0$ are values to be chosen later. Then, 
\begin{align*}
    \mathbb{P}_{\M'}^{\mathfrak{A}}\{\mathcal{E}\} &=
    \mathbb{P}_{\M'}^{\mathfrak{A}}\left\{N_{K,h}(s) \geq  \rho_{\M,h}^\star(s)K - f_{\alpha}(K), \sum_{a\in\O_{\M',h}(s)}N_{K,h}(s,a) < K^{\alpha + \beta}\right\}
    \\ &\leq \mathbb{P}_{\M'}^{\mathfrak{A}}\left\{N_{K,h}(s) - \sum_{a\in\O_{\M',h}(s)}N_{K,h}(s,a) > \rho_{\M,h}^\star(s)K - f_{\alpha}(K) - K^{\alpha + \beta}\right\} 
    \\ &\stackrel{(a)}{\leq} \frac{\E_{\M'}^{\mathfrak{A}}[\sum_{a\notin\O_{\M',h}(s)} N_{K,h}(s,a)]}{\rho_{\M,h}^\star(s)K - f_{\alpha}(K) - K^{\alpha + \beta}} \stackrel{(b)}{\leq} \frac{c(\M')AK^\alpha/\Delta_{\M'}}{\rho_{\M,h}^\star(s)K - f_{\alpha}(K) - K^{\alpha + \beta}},
\end{align*}
where (a) is from Markov inequality and (b) from Lemma \ref{lemma:bound-visits-consistent}. Similarly,
\begin{align*}
    \mathbb{P}_{\M}^{\mathfrak{A}}\{\mathcal{E}\} &= 1 - \mathbb{P}_{\M}^{\mathfrak{A}}\left\{N_{K,h}(s) <  \rho_{\M,h}^\star(s)K - f_{\alpha}(K) \vee \sum_{a\in\O_{\M',h}(s)}N_{K,h}(s,a) \geq K^{\alpha + \beta}\right\}
    \\ &\stackrel{(c)}{\geq} 1 - \mathbb{P}_{\M}^{\mathfrak{A}}\left\{N_{K,h}(s) <  \rho_{\M,h}^\star(s)K - f_{\alpha}(K) \right\} - \mathbb{P}_{\M}^{\mathfrak{A}}\left\{ \sum_{a\in\O_{\M',h}(s)}N_{K,h}(s,a) \geq K^{\alpha + \beta}\right\}
    \\ &\stackrel{(d)}{\geq} 1 - \frac{1}{K^\xi} - \frac{\E_{\M}^{\mathfrak{A}}[\sum_{a\in\O_{\M',h}(s)}N_{K,h}(s,a) ]}{K^{\alpha + \beta}}
    \stackrel{(e)}{\geq} 1 - \frac{1}{T^\xi} - \frac{{c}(\M)AK^\alpha/\Delta_\M}{K^{\alpha + \beta}},
\end{align*}
where (c) is from the union bound, (d) is from Lemma \ref{lemma:state-dist-consistent-prob} and Markov's inequality, and (e) is once again from Lemma \ref{lemma:bound-visits-consistent}. Choosing $0 < \beta < 1-\alpha$ and $0 < \xi < 1-\alpha$, we have that $\mathbb{P}_{\M}^{\mathfrak{A}}\{\mathcal{E}\} \rightarrow 1$ as $K\rightarrow\infty$ and $1/ \mathbb{P}_{\M'}^{\mathfrak{A}}\{\mathcal{E}\} \geq {O}(K^{1-\alpha})$. Therefore,
\begin{align}\label{eq:case1-lb}
    \liminf_{K \rightarrow \infty} \frac{\mathbb{P}_{\M}^{\mathfrak{A}}\{\mathcal{E}\}\log \frac{1}{\mathbb{P}_{\M'}^{\mathfrak{A}}\{\mathcal{E}\}} - \log 2}{\log(K)} \geq 1-\alpha.
\end{align}
Noting that this argument holds for all $\M'\in\alt$ concludes the proof.
\end{proof}

We are now ready to prove Theorem \ref{th:main-lower-bound}.

\paragraph{Proof of Theorem \ref{th:main-lower-bound}.}

Using Proposition \ref{prop:regret-decomp}, we can write the ``asymptotic regret'' of $\mathfrak{A}$ as
\begin{align}\label{eq:asym-regret}
   \liminf_{K\rightarrow \infty} \sum_{s\in\S}\sum_{a\in\A}\sum_{h\in[H]} \frac{\E_\M^\mathfrak{A}[N_{K,h}(s,a)]}{\log(K)}\Delta_{\M,h}(s,a).
\end{align}
Note that the expected visits $\liminf_{K\rightarrow \infty}\frac{\E_\M^\mathfrak{A}[N_{K,h}(s,a)]}{\log(K)}$ must satisfy the dynamical constraints of the underlying MDP $\M$. That is, for any state $s\notin\mathrm{supp}(p_0)$ and any $K > 0$,
\begin{align}\label{eq:lb-constr-h1}
    \sum_{a\in\A}\frac{\E_\M^\mathfrak{A}[N_{K,1}(s,a)]}{\log(K)} = 0.
\end{align}
Similarly, for $h>1$,
\begin{align}\label{eq:lb-constr-p}
   \sum_{a\in\A}\frac{\E_\M^\mathfrak{A}[N_{K,h}(s,a)]}{\log(K)} = \frac{\E_\M^\mathfrak{A}[N_{K,h}(s)]}{\log(K)} =  \sum_{s'\in\S}\sum_{a'\in\A}p(s|s',a')\frac{\E_\M^\mathfrak{A}[N_{K,h-1}(s',a')]}{\log(K)}.
\end{align}
The result follows by introducing the optimization variables $\eta_h(s,a) = \liminf_{K\rightarrow \infty}\frac{\E_\M^\mathfrak{A}[N_{K,h}(s,a)]}{\log(K)}$ and minimizing \eqref{eq:asym-regret} subject to the dynamical constraints \eqref{eq:lb-constr-h1} and \eqref{eq:lb-constr-p}, and to the information constraint of Theorem \ref{th:constr-lower-bound}.
\section{Analysis of the Lower Bound}

So far we derived the lower bound for a generic set of realizable MDPs $\mathfrak{M}$, which can be used to encode any structure of interest. In this section, we shall analyze the lower bound for the standard case of unstructured MDPs, i.e., where transition probabilities and rewards at different state-action-stage triplets are unrelated. We shall thus define $\mathfrak{M}$ as the set of MDPs whose transition probilities $p_h'(s,a)$ are any valid probability distribution over the state space $\S$ (i.e., $\mathbb{P}(\S)$) for each $s,a,h$. Similarly, we shall assume the rewards $q_h'(s,a)$ are given by any distribution supported on $[0,1]$ (i.e., $\mathbb{P}([0,1])$)\footnote{This is without loss of generality, as one could simply take the whole set of reals or any other closed subset} for our general computation, while we shall also instantiate the lower bound for Gaussian rewards in specific examples.

\subsection{Taxonomy of Alternative MDPs}\label{app:taxonomy-alt}

We define some useful sub/super-sets of the alternative MDPs $\alt$.

\begin{itemize}
    \item \emph{weak alternative} (WA): MDPs where every optimal policy for $\M$ is sub-optimal
    \begin{align*}
         \Lambda^{\mathrm{wa}}(\M) := \{\M'\in\mathfrak{M} \ |\ \weakopt(\M) \cap \weakopt(\M') = \emptyset \}.
    \end{align*}
    \item \emph{weakly confusing} (WC): MDPs that are indistinguishable from $\M$ by playing only optimal policies for $\M$
        \begin{align*}
         \Lambda^{\mathrm{wc}}(\M) := \{\M'\in\mathfrak{M} \ |\ \forall  s,a,h \in \mathcal{O}_{\M}^\star : \mathrm{KL}_{s,a,h}(\M,\M') = 0 \}.
    \end{align*}
    \item \emph{alternative}: MDPs where every optimal policy for $\M$ is sub-optimal and that are indistinguishable from $\M$ by playing only optimal policies for $\M$
        \begin{align*}
         \Lambda(\M) &:= \Lambda^{\mathrm{wa}}(\M) \cap \Lambda^{\mathrm{wc}}(\M)\\ &=  \{\M'\in\mathfrak{M} \ |\ \strongopt(\M) \cap \weakopt(\M') = \emptyset \} \cap \Lambda^{\mathrm{wc}}(\M). \qquad \text{(by Lemma \ref{lemma:equiv-weak-alt})}
    \end{align*}
    \item \emph{strongly confusing} (SC): MDPs that are indistinguishable from $\M$ by playing only optimal actions for the latter even with the possibility to reset the system in any state
        \begin{align*}
         \Lambda^{\mathrm{sc}}(\M) := \{\M'\in\mathfrak{M} \ |\ \forall  s,h,a \in \mathcal{O}_{\M,h}(s) : \mathrm{KL}_{s,a,h}(\M,\M') = 0 \}.
    \end{align*}
    \item \emph{strong alternative} (SA): all alternatives that are also strongly confusing
        \begin{align*}
         \Lambda^{\mathrm{sa}}(\M) := \Lambda(\M) \cap \Lambda^{\mathrm{sc}}(\M).
    \end{align*}
    \item \emph{decoupled strong alternative} (DSA): the subset of strong alternatives that yield a decoupled exploration strategy. These are such that, for any state not in the support of $\rho_\M^\star$, there exists an optimal action of $\M$ that remains optimal 
        \begin{align*}
         \Lambda^{\mathrm{dsa}}(\M) := \left\{\M'\in\Lambda^{\mathrm{sa}}(\M) \ |\  \forall h,s\notin\mathrm{supp}(\rho_{\M,h}^\star), \exists a^\star\in\mathcal{O}_{\M,h}(s) \cap \argmax_{a\in\A}\left\{ {r}_h'(s,a) +  {p}_h'(s,a)^T V_{\M,h+1}^\star\right\} \right\}.
    \end{align*}
    
\end{itemize}

\paragraph{Properties.}

We prove some useful properties for these sets. First, we show a necessary condition for an MDP to be a strong alternative.

\begin{lemma}[Property of strong alternatives]\label{lemma:strong-alt-properties}
Let $\M' \in \Lambda^{\mathrm{sa}}(\M)$ be any strong alternative to $\M$ (see Appendix \ref{app:taxonomy-alt}). Then, there exists a stage $h\in[H]$, a state $s\in\S$, and a sub-optimal action $a\notin\O_{\M,h}(s)$ such that
\begin{align*}
    r_h'(s,a) + p_h'(s,a)^T V_{\M,h+1}^\star > V_{\M,h}^\star(s).
\end{align*}
\end{lemma}
\begin{proof}
This is easily proved by contradiction. Suppose the claim does not. This means that
\begin{align}\label{eq:proof-strong-alt-1}
    \forall h\in[H], s\in\S, a\notin\O_{\M,h}(s) : r_h'(s,a) + p_h'(s,a)^T V_{\M,h+1}^\star \leq V_{\M,h}^\star(s).
\end{align}
Take any Bellman-optimal policy $\pi^\star \in \strongopt(\M)$ for $\M$. Note that, since $\M'$ is strongly confusing w.r.t. $\M$, the kernels of the two MDPs are identical in all actions chosen by $\pi^\star$. Therefore, $V_{\M',h}^{\pi^\star}(s) = V_{\M,h}^{\pi^\star}(s) = V_{\M,h}^{\star}(s)$ holds for all $s,h$. Moreover, by the Bellman optimality equations for $\M$,
\begin{align*}
    \forall h\in[H], s\in\S, a\in\O_{\M,h}(s) : r_h'(s,a) + p_h'(s,a)^T V_{\M,h+1}^\star = r_h(s,a) + p_h(s,a)^T V_{\M,h+1}^\star = V_{\M,h}^\star(s).
\end{align*}
Combining this with \eqref{eq:proof-strong-alt-1}, we get that
\begin{align*}
    \forall h\in[H], s\in\S : r_h'(s,\pi_h^\star(s)) + p_h'(s,\pi_h^\star(s))^T V_{\M',h+1}^{\pi^\star} = \max_{a\in\A} \{ r_h(s,a) + p_h(s,a)^T V_{\M',h+1}^{\pi^\star} \}.
\end{align*}
This implies that $\pi^\star$ satisfies the Bellman optimality equations in $\M'$, thus contradicting the fact that $\M'$ is alternative.
\end{proof}

Then, we prove a crucial result for the successive parts of this section. The following is a generalization of the decoupling lemma of \citet{ok2018exploration}, originally derived for ergodic average-reward MDPs, to finite-horizon MDPs without any reachability assumption.

\begin{lemma}\label{lemma:decoupling}
Let $\mathcal{U}_1,\mathcal{U}_2$ be two non-overlapping non-empty subsets of state-action-stage triplets such that, for all $(s,a,h)\in\mathcal{U}_0 := \mathcal{U}_1 \cup \mathcal{U}_2$, $a \notin \mathcal{O}_{\M,h}(s)$. Let $\bar{p},\bar{q}$ be any transition-reward kernels and,
for $i\in\{0,1,2\}$, define three MDPs $\{\M_i\}$ whose transition-reward kernels $\{p^i,q^i\}$ are given by
\begin{align*}
(p_h^i(s,a),q_h^i(s,a)) = 
    \begin{cases}
    (\bar{p}_h(s,a),\bar{q}_h(s,a)) \quad &\text{if}\ (s,a,h) \in \mathcal{U}_i,\\
    (p_h(s,a),q_h(s,a)) & \text{otherwise}.
    \end{cases}
\end{align*}
Moreover, suppose that $\bar{p},\bar{q}$ are such that, for all $ h\in[H],s\notin\mathrm{supp}(\rho_{\M,h}^\star)$,
\begin{align}\label{eq:decoupling-lemma-not-visited}
    \exists a^\star\in\mathcal{O}_{\M,h}(s) : {r}_h(s,a^\star) +  {p}_h(s,a^\star)^T V_{\M,h+1}^\star
    \geq \max_{a\in\A}\left\{ {r}_h^0(s,a) +  {p}_h^0(s,a)^T V_{\M,h+1}^\star\right\}.
\end{align}
Then, if $\M_0 \in \alt$, either $\M_1\in\alt$ or $\M_2\in\alt$.
\end{lemma}

A more intuitive explanation of this result is due before presenting the proof. Here each MDP $\M_i$ is defined starting from our true MDP $\M$ and changing only the transition-reward kernels (by setting them to $\bar{p},\bar{q}$) in state-action-stage triplets in $\mathcal{U}_i$. Moreover, MDP $\M_0$, by definition, is a \emph{decoupled strong alternative}.
What the lemma essentially proves is that, if $\M_0$ is indeed a decoupled strong alternative obtained from $\M$ by changing only the kernels in $\mathcal{U}_0$, then we could obtain another decoupled strong alternative (say, $\M_1$) by changing only the kernels in a strict subset $\mathcal{U}_1 \subset \mathcal{U}_0$. In the next section, we shall see that this significantly simplies the computation of the infimum in the KL constraint over these MDPs, which can be reduced to modifying a single state-action-stage triplet from $\M$.
\begin{proof}
Note that each $\M_i$ is confusing w.r.t. $\M$ since the kernels are unchanged at all optimal actions of the latter. Thus, $\M_i\in\alt$ holds if and only if $V_{\M_i,0}^\pi < V_{\M_i,0}^\star$ for all policies $\pi\in\strongopt(\M)$. By contradiction, let us assume that $\M_0 \in\alt$ but $\M_1\notin\alt$ and $\M_2\notin\alt$. This is equivalent to assuming that the following three conditions hold:
\begin{align}
    \forall \pi\in\strongopt(\M) &: V_{\M_0,0}^{\pi} < V_{\M_0,0}^\star \label{eq:decoupling-lemma-cond-1}
    \\   \exists \pi\in\strongopt(\M) &: V_{\M_1,0}^{\pi} = V_{\M_1,0}^\star \label{eq:decoupling-lemma-cond-2}
        \\   \exists \pi\in\strongopt(\M) &: V_{\M_2,0}^{\pi} = V_{\M_2,0}^\star \label{eq:decoupling-lemma-cond-3}
\end{align}
Since the Markov reward process induced by each $\pi\in\strongopt(\M)$ is exactly the same in the three MDPs and in $\M$ (see also Lemma \ref{lemma:alt-properties}), we have that
\begin{align}\label{eq:decoupling-equal-values}
    \forall \pi\in\strongopt(\M), i\in\{0,1,2\}, h \geq 0 :\  V_{\M_i,h}^{\pi} = V_{\M,h}^{\pi} = V_{\M,h}^{\star}.
\end{align}
Thanks to \eqref{eq:decoupling-equal-values}, \eqref{eq:decoupling-lemma-cond-2} and \eqref{eq:decoupling-lemma-cond-3} can actually be merged into the condition 
\begin{align*}
    \forall \pi\in\strongopt(\M) : V_{\M_1,0}^{\pi} = V_{\M_2,0}^{\pi} = V_{\M_1,0}^\star = V_{\M_2,0}^\star.
\end{align*}
Therefore, to find a contradiction, it is sufficient to prove that some policy $\pi\in\strongopt(\M)$ is simultaneously optimal in all three MDPs.

Let us take any state-action-stage triplet $(s,a,h)$ such that $s \in \mathrm{supp}(\rho_{\M,h}^\star)$. Then, if $(s,a,h) \notin \mathcal{U}_2$,
\begin{align*}
    r_h^0(s,\pi_h(s)) +  p_h^0(s,\pi_h(s))^T V_{\M_0,h+1}^\pi
    &\stackrel{(a)}{=} r_h^1(s,\pi_h(s)) +  p_h^1(s,\pi_h(s))^T V_{\M_1,h+1}^\pi
    \\ &\stackrel{(b)}{\geq} r_h^1(s,a) +  p_h^1(s,a)^T V_{\M_1,h+1}^\pi
    \\ &\stackrel{(c)}{=}  r_h^0(s,a) +  p_h^0(s,a)^T V_{\M_0,h+1}^\pi,
\end{align*}
where (a) follows from \eqref{eq:decoupling-equal-values} and the fact that the kernels are all equal in actions prescribed by $\pi$, (b) uses Lemma \ref{lemma:opt-act-vs-rho} and the fact that $\pi$ is optimal in $\M_1$, and (c) uses the equivalence of the kernels of $\M_0$ and $\M_1$ in all triplets $(s,a,h)$ not in $\mathcal{U}_2$. Using analogous steps, we can show that, if $(s,a,h) \in \mathcal{U}_2$,
\begin{align*}
    r_h^0(s,\pi_h(s)) +  p_h^0(s,\pi_h(s))^T V_{\M_0,h+1}^\pi
    &\stackrel{}{=} r_h^2(s,\pi_h(s)) +  p_h^2(s,\pi_h(s))^T V_{\M_2,h+1}^\pi
    \\ &\stackrel{}{\geq} r_h^2(s,a) +  p_h^2(s,a)^T V_{\M_2,h+1}^\pi
    \\ &\stackrel{}{=}  r_h^0(s,a) +  p_h^0(s,a)^T V_{\M_0,h+1}^\pi,
\end{align*}
Combining these two, we conclude that $\pi$ satisfies the Bellman optimality equations in $\M_0$ for all $h\in[H]$ and $s\in\mathrm{supp}(\rho_{\M,h}^\star)$, i.e.,
\begin{align*}
    \forall h, s\in\mathrm{supp}(\rho_{\M,h}^\star) : r_h^0(s,\pi_h(s)) +  p_h^0(s,\pi_h(s))^T V_{\M_0,h+1}^\pi
    \geq \max_{a\in\A}\left\{ r_h^0(s,a) +  p_h^0(s,a)^T V_{\M_0,h+1}^\pi\right\}.
\end{align*}
We now note that the same holds for all the states that are not in $\mathrm{supp}(\rho_{\M,h}^\star)$ by the condition \eqref{eq:decoupling-lemma-not-visited} on $\bar{p},\bar{q}$. Too see this, note that the policy $\pi\in\strongopt(\M)$ was arbitrary so far. Take $\pi\in\strongopt(\M)$ such that, for all $h\in[H], s\notin\mathrm{supp}(\rho_{\M,h}^\star)$, $\pi_h(s) = a^\star_h(s)$ for some optimal action $a^\star_h(s) \in \O_{\M,h}(s)$ that verifies \eqref{eq:decoupling-lemma-not-visited}. Since $V_{\M,h+1}^\star = V_{\M_0,h+1}^\pi$, \eqref{eq:decoupling-lemma-not-visited} yields
\begin{align*}
    \forall h, s\notin\mathrm{supp}(\rho_{\M,h}^\star) : r_h^0(s,\pi_h(s)) +  p_h^0(s,\pi_h(s))^T V_{\M_0,h+1}^\pi
    \geq \max_{a\in\A}\left\{ r_h^0(s,a) +  p_h^0(s,a)^T V_{\M_0,h+1}^\pi\right\}.
\end{align*}
Hence, we get that $\pi$ satisfies the Bellman optimality equations at all states and stages for $\M_0$. Thus, $\pi\in\weakopt(\M_0)$ and we get a contradiction. Therefore, it must be that $\pi$ is sub-optimal in either $\M_1$ or $\M_2$, which implies that either $\M_1 \in \alt$ or $\M_2 \in \alt$. This concludes the proof.
\end{proof}

\subsection{Computing the Lower Bound: General Case}\label{app:computing-general}

We show how to simplify the lower bound so as to better understand what are the main challenges behind its computation in the general case. Clearly, the infimum in the KL constraint represents the most challenging component, while the dynamics constraint is a simple linear constraint in $\eta$. We shall thus focus on computing the infimum over the alternatives $\alt$. In particular, we shall partition the set $\alt$ in three sub-sets on which the infimum yields (1) a closed-form solution in the first set, (2) a convex optimization problem in the second set, and (3) a non-convex optimization problem in the third set.

We partition the set of alternatives into the following non-overlapping subsets:
        \begin{align*}
         \Lambda(\M) := \underbrace{\Lambda^{\mathrm{dsa}}(\M)}_{(a)} \cup \underbrace{\big(\Lambda^{\mathrm{sa}}(\M) \setminus \Lambda^{\mathrm{dsa}}(\M)\big)}_{(b)} \cup \underbrace{\big(\Lambda(\M) \setminus \Lambda^{\mathrm{sa}}(\M)\big)}_{(c)}.
    \end{align*}
Thus, the infimum in the constraint can be reduced to a minimum of the infima over these three sets. We analyze them separately.

\paragraph{Set (a).}

Using the decoupling lemma (see Lemma \ref{lemma:decoupling} and comments below it), it is easy to see that the infimum over the set of \emph{decoupled strong alternatives} is attained at single state-action-stage triplets. This is because, if we create a decoupled strong alternative by changing the kernels of $\M$ at more than one state-action-stage triplets, then Lemma \ref{lemma:decoupling} guarantees that we can create another decoupled strong alternative by changing the kernels of $\M$ at fewer triplets by the same amount (thus with lower KL divergence). Moreover, Lemma \ref{lemma:strong-alt-properties} ensures that, if $\M'$ is a decoupled strong alternative created from $\M$ by changing the single triplet $(s,a,h)$, then
\begin{align*}
    r_h'(s,a) + p_h'(s,a)^T V_{\M,h+1}^\star > V_{\M,h}^\star(s).
\end{align*}
Thus, let us define the following local measure of complexity to discriminate between MDP models: 
\begin{align*}
    \mathcal{K}_{s,a,h}(\M) : = \inf_{\bar{p},\bar{q} \in \Lambda_{s}(\M)}\big\{ \mathrm{KL}(p_h(s,a), \bar{p}) + \mathrm{KL}(q_h(s,a), \bar{q})\big\}
\end{align*}
where $\Lambda_{s}(\M) := \{\bar{p}\in\P(\S),\bar{q}\in\P([0,1]) : \mathbb{E}_{x\sim \bar{q}}[x] + \bar{p}^T V_{\M,h+1}^\star > V_{\M,h}^\star(s)\}$. \footnote{If the set $\Lambda_s(\M)$ is empty we set the infimum to $+\infty$.} Then,
\begin{align*}
      \inf_{\M' \in \Lambda^{\mathrm{dsa}}(\M) } \sum_{s\in\S}\sum_{a\in\A} \sum_{h\in[H]}\eta_h(s,a)\mathrm{KL}_{s,a,h}(\M,\M')
      = \min_{h\in[H]}\min_{s\in\mathrm{supp}(\rho_{\M,h}^\star)}\min_{a\notin\O_{\M,h}(s)}\eta_h(s,a)\mathcal{K}_{s,a,h}(\M).
\end{align*}
Notably, this implies that the visits to each $(s,a,h)$ required by this part of the KL constraint are \emph{decoupled}, in the sense that they depend only on local quantities of such triplet $(s,a,h)$.

\paragraph{Set (b).}

In this case the infimum is not available in closed-form anymore, but the corresponding optimization problem is convex. Note that all alternative MDPs in set (b) are strongly confusing, which means that Lemma \ref{lemma:strong-alt-properties} still applies. Then, one can rewrite the optimization problem over this set in a form that depends on the local kernels $p_h'(s,a), q_h'(s,a)$ (i.e., the optimization variables over which we are computing the infimum) and on the optimal value function $V_{\M,h}^\star$ of $\M$ (which is fixed and does not depend on such optimization variables). Unfortunately, a closed-form expression is not available in general since the kernels $p_h'(s,a), q_h'(s,a)$ of the alternative MDP $\M'$ realizing the infimum might be obtained by carefully changing the kernels of $\M$ at multiple state-action-stage triplets.

\paragraph{Set (c).}

Computing the infimum over this set yields a non-convex optimization problem in general. This can be shown by using the same counter-example as the one recently proposed by \citet{marjani2021adaptive} for the best-policy identification setting. The intuition is that, since MDPs in this set are only weakly (and not strongly) confusing, the value functions of the optimal policies of $\M$ when evaluated in $\M'$ might differ from those in $\M$ itself. The definition of the alternative MDPs can be equivalently written with terms of the form $p_h'(s,a)^T V_{\M',h}^{\pi^\star}$, where $\pi^\star$ is an optimal policy for $\M$. These are bilinear in the optimization variables ($p',q'$) and thus hard to optimize over in general.

\subsection{The Full-Support Case (Proof of Proposition \ref{prop:full-support})}

We now consider the case where $\rho_{\M}^\star$ is full-support, i.e., $\rho_{\M,h}^\star(s) > 0$ for all $s\in\S,h\in[H]$. Note that a similar assumption has been considered by~\citet{ramponi2021online}, which is refered to as a form of \emph{weak ergodicity} by the authors. In this case, it is easy to see that the subsets of alternatives (b) and (c) in the decomposition of Appendix \ref{app:computing-general} are empty. Therefore, $\alt = \Lambda^{\mathrm{dsa}}(\M)$, and the infimum of the whole KL constraint can be computed in closed-form as
\begin{align*}
      \inf_{\M' \in \alt } \sum_{s\in\S}\sum_{a\in\A} \sum_{h\in[H]}\eta_h(s,a)\mathrm{KL}_{s,a,h}(\M,\M')
      = \min_{h\in[H]}\min_{s\in\S}\min_{a\notin\O_{\M,h}(s)}\eta_h(s,a)\mathcal{K}_{s,a,h}(\M).
\end{align*}
Plugging this into the main optimization problem and dropping the dynamics constraint, we get the following closed-form lower bound
\begin{equation}\label{eq:lower-bound-full-support}
v^\star(\M) = (1-\alpha)\sum_{h\in[H]}\sum_{s\in\S}\sum_{a\notin\O_{\M,h}(s)} \frac{\Delta_{\M,h}(s,a)}{\mathcal{K}_{s,a,h}(\M)}.
\end{equation} 

\paragraph{The dynamics constraint does not matter.}

It only remains to prove that the dynamics constraint indeed does not matter in the full-support case and can be dropped without changing the lower bound value. This is better seen by switching to the policy-based perspective. Recall that an allocation $\eta$ satisfies the dynamics constraint if, and only if, there exists a vector $\omega \in \mathbb{R}^{|\Pi|}_{\geq 0}$ such that $\eta_h(s,a) = \sum_{\pi\in\Pi} \omega_\pi \rho_h^\pi(s,a)$ for all $s,a,h$. Now take any $s,h$ and sub-optimal action $a\notin\O_{\M,h}(s)$. Recall that the KL constraint requires $\eta_h(s,a) \geq 1/\mathcal{K}_{s,a,h}(\M)$. Build a policy $\pi$ which is equal to an optimal policy $\pi^\star$ everywhere except for $\pi_{h}(s)=a$ and set $\omega_\pi = \frac{1}{\rho_{\M,h}^\star(s)\mathcal{K}_{s,a,h}(\M)}$. Repeat the same trick for all stages, states, and sub-optimal actions. Clearly, the resulting $\eta_h(s,a)$ satisfies the dynamics constraint because it is written as a mixture of deterministic policies. Moreover, it allocates exactly the number of pulls required by the KL constraint (this is because each chosen policy selects only a single sub-optimal action). Therefore, the final lower bound is exactly the one of \eqref{eq:lower-bound-full-support}.

\paragraph{Relation between $\mathcal{K}$ and $\Delta$.}

Note that the constraint in the definition of the local alternative set $\Lambda_{s,a,h}(\M)$ can be equivalently written as
\begin{align*}
    \big(\mathbb{E}_{x\sim \bar{q}}[x] - \mathbb{E}_{x\sim q_h(s,a)}[x]\big) + \big(\bar{p} - p_h(s,a)\big)^T V_{\M,h+1}^\star > \Delta_{\M,h}(s,a),
\end{align*}
where we simply added and subtracted $Q_{\M,h}^\star(s,a)$ from both sides. By upper bounding the left-hand side, this condition implies
\begin{align*}
    \|\bar{q}-q_h(s,a)\|_1 + (H-h)\|\bar{p}-p_h(s,a)\|_1 > \Delta_{\M,h}(s,a).
\end{align*}
Therefore, by Pinsker's inequality,
\begin{align*}
    \mathcal{K}_{s,a,h}(\M) &= \inf_{\bar{p},\bar{q} \in \Lambda_{s,a,h}(\M)}\big\{ \mathrm{KL}(p_h(s,a), \bar{p}) + \mathrm{KL}(q_h(s,a), \bar{q})\big\}
    \\ & \geq \frac{1}{2}\inf_{\bar{p},\bar{q} \in \Lambda_{s,a,h}(\M)}\big( \|\bar{q}-q_h(s,a)\|_1 + \|\bar{p}-p_h(s,a)\|_1 \big)^2
    \\ & \geq \frac{\Delta_{\M,h}(s,a)^2}{2(H-h)^2}.
\end{align*}
Thus, an upper bound to the lower bound of \eqref{eq:lower-bound-full-support} is
\begin{align*}
    v^\star(\M) \leq \sum_{h\in[H]}\sum_{s\in\S}\sum_{a\notin\O_{\M,h}(s)} \frac{2(H-h)^2}{\Delta_{\M,h}(s,a)},
\end{align*}
which proves Proposition \ref{prop:full-support}.

\subsection{The Impact of the Dynamics Constraint}\label{app:impact-dynamics}

This section proves the following general result: the optimal value of the lower bound Theorem \ref{th:main-lower-bound} \emph{without} the dynamics constraint is always upper bounded by the sum of local complexity measures (i.e., by the lower bound obtained when $\rho_\M^\star$ is full-support). By the examples in Appendix \ref{app:examples}, we know that this can be arbirarily loose, hence proving the importance of including the dynamics constraint.
\begin{theorem}
Let $\wt{v}(\M)$ be the optimal value of the optimazion problem in Theorem \ref{th:main-lower-bound} without the dynamics constraint. Then, for any MDP $\M$,
\begin{align*}
    \wt{v}(\M) \leq (1-\alpha) \sum_{h\in[H]}\sum_{s\in\S}\sum_{a\notin\O_{\M,h}(s)}\frac{\Delta_{\M,h}(s,a)}{\mathcal{K}_{s,a,h}(\M)}.
\end{align*}
\end{theorem}
\begin{proof}
We first show that, when the dynamics constraints are dropped, any MDP that is not a \emph{strong alternative} (i.e., set (c) in the general decomposition of Appendix \ref{app:computing-general}) can be discriminated from $\M$ by suffering zero regret. Take any $\M' \in \Lambda(\M) \setminus \Lambda^{\mathrm{sa}}(\M)$. Then, by definition, there exist $(s,a,h)\in\S\times\A\times[H]$ with $s\notin\mathrm{supp}(\rho_{\M,h}^\star)$ and $a\in\O_{\M,h}(s)$ such that
\begin{align*}
    \mathrm{KL}_{s,a,h}(\M,\M') \geq \epsilon,
\end{align*}
for some positive constant $\epsilon > 0$. Since such state-action-stage triplet does not appear in the objective value (i.e., it contributes to zero regret regardless of its number of visits), one can choose $\eta_h(s,a) \geq \frac{1-\alpha}{\epsilon}$ and satisfy the KL constraint. Repeating the previous argument for all $\M' \in \Lambda(\M) \setminus \Lambda^{\mathrm{sa}}(\M)$ shows that all MDPs that are not strong alternatives can be safely dropped from the constraint as they can be discriminated by suffering zero regret.

Let us now take any MDP $\M'$ in the remaining set of alternatives, i.e., $\M' \in \Lambda^{\mathrm{sa}}(\M)$. Then, $\M'$ is a strong alternative and Lemma \ref{lemma:strong-alt-properties} ensures that there exists a stage $h\in[H]$, state $s\in\S$, and sub-optimal action $a\notin\O_{\M,h}(s)$ such that
\begin{align*}
    \mathbb{E}_{x\sim q'_h(s,a)}[x] + p_h'(s,a)^T V_{\M,h+1}^\star > V_{\M,h}^\star(s).
\end{align*}
Then, this implies that
\begin{align*}
    \Lambda^{\mathrm{sa}}(\M) \subseteq \bigcup_{h\in[H]}\bigcup_{s\in\S}\bigcup_{a\notin\O_{\M,h}(s)}\Lambda_{s,a,h}(\M).
\end{align*}
Therefore, one can replace the alternative set $\Lambda^{\mathrm{sa}}(\M)$ with the one above in the KL constraint, which can only increase the optimal objective value since a (possibly) stronger constraint is imposed. Then, proceeding exactly as in the computation of the closest alternative for set $(a)$ in Appendix \ref{app:computing-general}, one can simplify this new constraint to
\begin{align*}
    \min_{h\in[H]}\min_{s\in\S}\min_{a\notin\O_{\M,h}(s)}\eta_h(s,a)\mathcal{K}_{s,a,h}(\M) \geq 1-\alpha.
\end{align*}
Plugging the resulting minimal values for $\eta$ into the objective function concludes the proof.
\end{proof}

\subsection{The Policy-based Perspective (Proof of Proposition \ref{prop:policy-based})}\label{app:policy-based}

Recall that, from \citep[][Remark 6.1, page 64]{altman1999constrained}, an allocation $\eta$ satisfies the dynamics constraint if, and only if, there exists a vector $\omega \in \mathbb{R}^{|\Pi|}_{\geq 0}$ such that $\eta_h(s,a) = \sum_{\pi\in\Pi} \omega_\pi \rho_h^\pi(s,a)$ for all $s,a,h$. Then, in order to prove Proposition \ref{prop:policy-based}, it is sufficient to use the change of variables $\eta_h(s,a) = \sum_{\pi\in\Pi} \omega_\pi \rho_{h}^\pi(s,a)$ in the optimization problem of Theorem \ref{th:main-lower-bound} and apply Proposition \ref{prop:policy-vs-action-gap} to rewrite the objective function.

\subsection{The Case of Known Dynamics: Reduction to Combinatorial Semi-bandits}\label{app:known-dynamics}

In this section, we suppose that all MDPs in $\mathfrak{M}$ have the same transition probabilities as $\M$. Equivalently, $\mathfrak{M}$ is the set of realizable MDPs when the transition probabilities $\{p_h\}$ of $\M$ are known. Thus, we can drop the KL divergence between the transition kernels of $\M$ and those of its alternatives from the KL constraint. Moreover, suppose that the reward distributions of the MDPs in $\mathfrak{M}$ is Gaussian with arbitrary mean and unit variance. First, we note that a weak alternative $\M'\in\Lambda^{\mathrm{wa}}(\M)$ belongs to the alternatives $\alt$ if, and only if, there exists a sub-optimal policy $\pi\notin\weakopt(\M)$ such that, for all $\pi^\star\in\strongopt(\M)$, $V_{\M',0}^\pi > V_{\M',0}^{\pi^\star} = V_{\M,0}^\star$. 

We now rewrite the optimization problem in a more convenient form. Let $\theta \in \mathbb{R}^{SAH}$ be a vector containing all mean rewards, i.e., $\theta_{s,a,h} = r_h(s,a)$. Moreover, let $\phi^\pi \in \mathbb{R}^{SAH}$ be the vector containing the state-action distribution $\rho_{\M}^\pi$, i.e., $\phi_{s,a,h}^\pi = \rho_{\M,h}^\pi(s,a)$. Then, $V_{\M,0}^\pi = \theta^T \phi^\pi$ and $V_{\M,0}^\star = \theta^T \phi^{\star} := \theta^T \phi^{\pi^\star}$ for some optimal policy $\pi^\star$.

Let us use the policy-based perspective of Proposition \ref{prop:policy-based}. Then, the KL in the constraint can be rewritten as
\begin{align*}
\sum_{\pi\in\Pi} \omega_\pi \sum_{s\in\S}\sum_{a\in\A} \sum_{h\in[H]}\rho_{\M,h}^\pi(s,a)\mathrm{KL}_{s,a,h}(\M,\M') 
&= \frac{1}{2}\sum_{\pi\in\Pi} \omega_\pi \sum_{s\in\S}\sum_{a\in\A} \sum_{h\in[H]}\rho_{\M,h}^\pi(s,a)(\theta_{s,a,h}-\theta_{s,a,h}')^2
\\&= \frac{1}{2} \|\theta-\theta'\|_{D_\omega}^2,
\end{align*}
where $D_\omega := \sum_{\pi\in\Pi} \omega_\pi \mathrm{diag}(\rho_{\M,h}^\pi(s,a))$. Now fix any policy $\pi\notin\weakopt(\M)$. It is easy to prove that the minimum (in $\theta'$) of this quantity subject to $\langle \theta', \phi^\pi \rangle \geq V_{\M,0}^\star$ can be evaluated in closed-form as\footnote{The inverse of $D_\omega$ has to be intended as pseudo-inverse. So if $D_\omega$ has a zero on an element of its diagonal, $D_{\omega}^{-1}$ has a zero at the corresponding element.}
\begin{align*}
\min_{\theta' : \langle \theta', \phi^\pi \rangle \geq V_{\M,0}^\star}\|\theta-\theta'\|_{D_\omega}^2 = \frac{\Gamma_{\M}^2(\pi)}{\|\phi^\pi\|_{D_\omega^{-1}}^2}.
\end{align*}

This leads to the following optimization problem
\begin{equation*}
	\begin{aligned}
    &\underset{\omega_\pi \geq 0}{\inf}&& \sum_{\pi\in\Pi} \omega_\pi \langle \theta, \phi^\star - \phi^\pi\rangle
     \\
     &
    \quad \mathrm{s.t.} \quad
     && 
        \forall \pi \notin \weakopt(\M) : \|\phi^\pi\|_{D_\omega^{-1}}^2 \leq \frac{\Gamma(\pi)^2}{2(1-\alpha)}.
	\end{aligned}
	\end{equation*}
Notably, this is almost equivalent to the optimization problem of the asymptotic lower bound for combinatorial semi-bandits \citep[e.g.,][]{wagenmaker2021experimental}. The almost equivalence comes from the fact that, in combinatorial semi-bandits, the feature vectors usually take values in $\{0,1\}^d$, where $d$ is their dimension ($d=SAH$ in our case). Here, instead, they contain values in $[0,1]^{SAH}$ representing the probabilities that the corresponding policy visits each state-action pair. When the MDP is deterministic, the two problems are indeed equivalent. We remark that the learning feedback itself is the same as the one in combinatorial semi-bandits: whenever we execute a policy $\pi$ in a deterministic MDP, we receive a random reward observation for each state-action-stage triplet associated with an element where $\phi^\pi$ is equal to $1$. When the MDP is not deterministic, on the other hand, we receive the observation only with the corresponding probability.

\paragraph{When dropping the dynamics constraint.}

In our examples later on, we shall also use the variant of this optimization problem obtained by dropping the dynamics constraint. Using the same procedure as before, we can reformulate the optimization problem \emph{without} dynamics constraint as
\begin{equation*}
	\begin{aligned}
    &\underset{\eta_h(s,a) \geq 0}{\inf}&& \sum_{h\in[H]}\sum_{s\in\S}\sum_{a\notin\O_h(s)}\eta_h(s,a)\Delta_h(s,a)
     \\
     &
    \quad \mathrm{s.t.} \quad
     && 
        \forall \pi \notin \weakopt(\M) : \|\phi^\pi\|_{D_\omega^{-1}}^2 \leq \frac{\Gamma(\pi)^2}{2(1-\alpha)},
	\end{aligned}
	\end{equation*}
where $D_\omega := \mathrm{diag}(\eta_h(s,a))$ and $\phi^\pi$ has the same meaning as before. Note that the only difference w.r.t. the case with dynamics constraint is the matrix $D_\omega$.

\subsection{Remarks on How to Use the Lower Bound}

\paragraph{Computing the lower bound in specific instances.}

While computing the lower bound for a general MDP might be hard (especially due the subsets (b) and (c) of alternative models mentioned in Appendix \ref{app:computing-general}), it may still be possible to compute it (or to find an approximation of it) in specific MDP instances. The general remark is that ignoring some alternative MDPs in the computation can only decrease the optimal objective value and, thus, it always yields a valid lower bound. Although the resulting lower bound might not be tight in all its dependences, this intuition can be used to prove that certain dependences of interest are unavoidable in specific examples. This is indeed what we do in Appendix \ref{app:examples} to match the existing lower bound of \citet{xu2021fine} on their hard MDP instance. In that case, we drop all alternative MDPs with different transition probabilities than $\M$, which allows us to compute the lower bound in closed form and achieve the desired result. 

\paragraph{Simplifying the lower bound}

We do not exclude that some of the alternative MDPs in our set $\alt$ are redundant, i.e., that their impact on the optimal allocation is already covered by other alternative MDPs or by the dynamics constraint. In such a case, those MDPs could be safely dropped from the optimization problem, hence possibly simplifying its computation. Moreover, we wonder whether a decoupling result similar to Lemma \ref{lemma:decoupling} holds also for MDPs in the set (b) of Appendix \ref{app:computing-general} (i.e., strongly confusing MDPs that are not decoupled strong alternatives). For instance, while it is unlikely that the infimum is attained over a single state-action-stage triplet, it might be possible to show that it is attained over a (small) subset of all triplets.

\paragraph{Making the lower bound more interpretable.}

Another direction is to make the lower bound more interpretable by finding lower and upper bounds to the optimal objective value $v^\star(\M)$. For instance, these might be in the form of lower/upper bounds to the number of visits required by the optimal allocation to any state-action-stage triplet. This would help us in understanding how far existing algorithms are from being optimal in the general case, while possibly yielding intuition on how to design near-optimal strategies. As an example of this, \citet{marjani2021adaptive}, after showing that computing the optimization problem for best-policy identification might be intractable, derive a tractable (closed-form) upper bound to its optimal value and use that to design near-optimal algorithms. In our case, we believe that the optimal objective value could be upper bounded by a quantity that is roughly the regret upper bound of \citet{xu2021fine}. This would clearly demonstrate the impact of multiple optimal actions onto the lower bound and, in particular, how they make minimum gaps propagate to different states through the dynamics constraint. We shall explore this direction in future work.

\paragraph{Structured vs Unstructured MDPs.}

We derived the lower bound for a generic set of realizable MDPs $\mathfrak{M}$, while then we instatiated it for the standard case of unstructured MDPs (i.e., where transition probabilities and rewards at different state-action-stage triplets are unrelated). Anyway, one could use the set $\mathfrak{M}$ to instatiate the lower bound for any structure of interest. For instance, relevant structures studied in the literature are low-rank MDPs \citep{jin2020linear}, linear-mixture MDPs \citep{Ayoub2020vtr}, Lipschitz MDPs \citep{ok2018exploration}, etc.
\section{Examples with Specific MDP Instances}\label{app:examples}

\subsection{Hard MDP from \citet{xu2021fine} (Proof of Corollary \ref{cor:du.example})}\label{app:example-du}

Conside the MDP in Figure \ref{fig:example-du-variant} with $\kappa=0$. The MDP is a binary tree with depth $H$ where nodes (i.e., states) are connected by deterministic transitions given by the available actions. States are denoted by $s_i^j$, where $j\in[H]$ denotes the level of the tree and $i\in[n_j]$ denotes the index of the state among the $n_j := 2^{j-1}$ nodes at such level. Each state in levels from $1$ to $H-1$ has exactly two actions available: $L$ (left) and $R$ (right), with obvious meaning. On the other hand, states at level $H$ have exactly $m$ actions available, denoted by $a_1,a_2,\dots,a_m$. The mean reward is zero everywhere except for $r_H(s_1^H,a_1) = \epsilon > 0$. The reward distribution is Gaussian with unit variance. Note that the ``path'' of a policy is fully specified by exactly $H$ actions, one for each stage to be applied in the resulting (deterministic) state. The (unique) optimal path is ``always go left and take $a_1$ at the last stage'', i.e., any policy $\pi^\star = [L,L,\dots,L,a_1]$ is optimal. Note that this does mean that the optimal policy is unique: since the actions chosen in states that are not reached outside the optimal path do not matter, there exists a huge amount of optimal policies. Anyway, we can restrict the set of available deterministic policies $\Pi$ by simply taking a single policy for each deterministic path, yielding a total number of $2^{H-1}m$ policies.

\paragraph{Computing the exact lower bound.}

Recall that we suppose the transition probabilities are fixed and only deal with the rewards. We shall use the convenient reformulation of Appendix \ref{app:known-dynamics}. Take any sub-optimal policy $\pi \notin \weakopt(\M)$. Let $s_h^\pi$ and $a_h^\pi$ denote the deterministic state and action visited by policy $\pi$ at stage $h$, respectively. Note that the vector $\phi^\pi$ has exactly $H$ components equal to one (those corresponding to $h,s_h^\pi,a_h^\pi$), and all the others equal to zero. Moreover, $\Gamma(\pi) = \epsilon$ holds for all sub-optimal policies. Therefore, the constraint for $\pi$ can be rewritten as
\begin{align*}
    \sum_{h=1}^H \frac{1}{\sum_{\pi'\in\Pi}\alpha_{\pi'}\indi{\rho_h^{\pi'}(s_h^\pi,a_h^\pi)=1}} \leq \frac{\epsilon^2}{2(1-\alpha)}.
\end{align*}
Now fix some stage $h\in[H]$. We need to count how many policies/paths pass through $s_h^\pi,a_h^\pi$ (i.e., how many policies are active in the denominator above). Clearly, when $h=H$, there exists only one such policy, thus $\sum_{\pi'\in\Pi}\indi{\rho_H^{\pi'}(s_H^\pi,a_H^\pi)=1} = 1$. For $h<H$, this is exactly the total number of paths in the sub-tree reached by playing $a_h^\pi$ in $s_h^\pi$, hence $\sum_{\pi'\in\Pi}\indi{\rho_h^{\pi'}(s_h^\pi,a_h^\pi)=1} = 2^{H-h-1}m$. Moreover, note that the constraint above is perfectly symmetrical for different policies $\pi \notin \strongopt(\M)$, that is, each policy appears in the constraint of other policies the same number of times. Therefore, each policy must be played the same amount of times (say, $\eta$) in the optimal solution. This allows us to rewrite the constraint above as
\begin{align*}
    \forall \pi \notin \weakopt(\M) : \frac{1}{\eta}\left(\sum_{h=1}^{H-1} \frac{1}{2^{H-1-h}m} + 1 \right) \leq \frac{\epsilon^2}{2(1-\alpha)}.
\end{align*}
Let $\eta^\star$ be the minimal value that satisfies this constraint.
This yields the following optimal value for the optimization problem:
\begin{align*}
    v^\star(\M) &= \sum_{\pi\in\Pi\setminus\Pi^\star} \eta^\star \epsilon = \frac{2(1-\alpha)}{\epsilon} \left(\sum_{h=1}^{H-1} \frac{1}{2^{H-1-h}m} + 1 \right) \left(2^{H-1}m - 1\right)
    \\ &= \frac{2(1-\alpha)}{\epsilon} \left(\sum_{h=1}^{H-1}2^{h} + 2^{H-1}m -\sum_{h=1}^{H-1} \frac{1}{2^{H-1-h}m} - 1 \right).
\end{align*}
Since $S = \sum_{h=0}^{H-1} 2^h = 2^H - 1$ and $A = m$, it is easy to check that $\sum_{h=1}^{H-1}2^{h} = S - 1$ and $2^{H-1}m = \frac{A(S+1)}{2}$. Then, 
\begin{align*}
    v^\star(\M) = \frac{2(1-\alpha)}{\epsilon} \left(S - 2 + \frac{A(S+1)}{2} - 2 \frac{S-1}{A(S+1)}\right).
\end{align*}
Since we assume $S\geq 1$ and $A=m\geq 2$, this can be lower bounded by
\begin{align*}
    v^\star(\M) \geq (1-\alpha)\frac{SA}{\epsilon}.
\end{align*}
However, if we compute the sum of inverse (positive) sub-optimality gaps, we obtain a much smaller quantity (and thus a non-tight lower bound). To see this, note that the local sub-optimality gaps are zero for all states and actions except for the few states $s_1^1,s_1^2,\dots,s_1^H$ visited by the optimal policy and corresponding sub-optimal action. More precisely, there are in total $H-1$ gaps equal to $\epsilon$ (for playing action $R$ in one of the first $H-1$ states) and $m-1$ gaps equal to $\epsilon$ (for playing $a_j\neq a_1$ in state $s_1^H$). Since $H = \log_2(S+1)$ and $m=A$, the sum of inverse sub-optimality gaps is
\begin{align*}
    \sum_{h\in[H]}\sum_{s\in\S}\sum_{a\notin\O_h(s)} \frac{1}{\Delta_{h}(s,a)} = \frac{\log_2(S+1) + A - 2}{\epsilon},
\end{align*}
which is smaller than our lower bound by at least a factor $S/\epsilon$.

\paragraph{Computing the lower bound without dynamics constraint.}

Let us now compute the same lower bound without dynamics constraint. We use the variant of the optimization problem in Appendix \ref{app:known-dynamics}. Using the same intuition as in Appendix \ref{app:impact-dynamics}, we can set $\eta_h(s,a) = +\infty$ for all optimal actions (which are not part of the objective function). This implicitly satisfies the constraints for all policies that visit some state outside the optimal path $s_1^1,\dots,s_1^H$. Thus, we are only left to deal with policies that only visit the optimal path and take some sub-optimal action in one such state. Clearly, the tightest constraints are attained by policies that take a \emph{single} sub-optimal action along such path. Then, take a policy which takes a sub-optimal action $a$ in state $s_1^h$ (i.e., $a=R$ for $h<H$ and $a=a_j$ for $h=H$ and $j>1$). The constraint is simply
\begin{align*}
   \frac{1}{\eta_h(s_1^h,a)} \leq \frac{\Delta_h(s_1^h,a)^2}{2(1-\alpha)}.
\end{align*}
Repeating this argument for all states/stages along the optimal path and sub-optimal actions, we obtain the following value for the optimization problem
\begin{align*}
    v^\star(\M) = \sum_{h\in[H]}\sum_{a\notin\O_h(s_1^h)} \frac{2(1-\alpha)}{\Delta_{h}(s_1^h,a)} = \frac{2(1-\alpha)}{\epsilon}(H + m - 2) = 2(1-\alpha)\frac{\log_2(S+1) + A - 2}{\epsilon}.
\end{align*}

\subsection{Dependence on the Minimum Gap (Proof of Proposition \ref{prop:delta.max})}

We now consider the variant of the MDP in Figure \ref{fig:example-du-variant} with $\kappa \geq 2\epsilon > 0$ and show that the lower bound does not scale with $\epsilon$ anymore.

We follow exactly the same steps as in Appendix \ref{app:example-du} while this time we seek an upper bound to $v^\star(\M)$ which does not scale with $\epsilon$. Note that now the gap of each sub-optimal policy is $\kappa$, except for the previously-optimal policy (the one that always goes left) whose gap is $\kappa-\epsilon$. Since $\epsilon \leq \kappa/2$, this implies that the gap of all policies is at least $\kappa/2$. We can thus impose the stronger constraint
\begin{align*}
    \forall \pi \notin \weakopt(\M) : \sum_{h=1}^H \frac{1}{\sum_{\pi'\in\Pi}\alpha_{\pi'}\indi{\rho_h^{\pi'}(s_h^\pi,a_h^\pi)=1}} \leq \frac{\kappa^2}{8(1-\alpha)},
\end{align*}
where we simply replaced the gap of each policy $\pi$ on the righthand side by $\kappa/2$. We compute this exactly as in Example 1 and obtain
\begin{align*}
    v^\star(\M) &\leq \sum_{\pi\in\Pi\setminus\Pi^\star} \eta^\star \underbrace{\Gamma(\pi)}_{\leq \kappa} \leq \frac{8(1-\alpha)}{\kappa} \left(\sum_{h=1}^{H-1} \frac{1}{2^{H-1-h}m} + 1 \right) \left(2^{H-1}m - 1\right)
    \\ &= \frac{8(1-\alpha)}{\kappa} \left(S - 2 + \frac{A(S+1)}{2} - 2 \frac{S-1}{A(S+1)}\right)\leq 12(1-\alpha)\frac{SA}{\kappa},
\end{align*}
where the last inequality holds since $S \leq SA/2$.
This does not scale with $\epsilon$ as we wanted to prove. 

Note that the algorithm of \cite{xu2021fine} has a regret bound of order $\frac{SA}{\epsilon}\log(K)$ on this instance, which can be arbitrarily sub-optimal w.r.t. the lower bound derived above.
\section{Policy-Gap-Based Regret Bound for Optimistic Algorithms (Proof of Theorem \ref{lem:ucbvi.upper})}


We recall that we prove Theorem~\ref{th:ucbvi-prob-regret} for an MDP $\M$ with unknown rewards (assumed in $[0,1]$) and transition probabilities and, for simplicity, deterministic initial state $s_1$. Since the initial state is fixed, we shall write $V_1^\pi := V_1^\pi(s_1) = V_0^\pi$. The following lemma shows how to build the Chernoff-Hoeffding-based confidence bonus \citep{azar2017minimax}.

\begin{lemma}[Optimism]\label{lemma:ucbvi-optimism}
Consider the following choice for the optimistic bonuses:\footnote{We set $N_h^0(s,a)=0$ and assume the ratio to be equal to infinity when the state-action pair is not visited.}
\begin{align*}
    c_h^k(s,a) := H\sqrt{\frac{\log \frac{4SAHK}{\delta}}{2N_h^{k-1}(s,a)}} \wedge H.
\end{align*}
Then, with probability at least $1-\delta$, $\wb{Q}_h^k(s,a) \geq Q_h^\star(s,a)$ holds for all $s\in\S,a\in\A,h\in[H],k\in[K]$.
\end{lemma}
\begin{proof}
By the triangle inequality and Hoeffding's inequality,
\begin{align*}
    |r_h(s,a) + p_h(s,a)^T V_{h+1}^\star &- \wh{r}_h^k(s,a) - \wh{p}_h^k(s,a)^T V_{h+1}^\star|\\ &\leq |r_h(s,a) - \wh{r}_h^k(s,a)| + | p_h(s,a)^T V_{h+1}^\star -  \wh{p}_h^k(s,a)^T V_{h+1}^\star| \\ & \leq \sqrt{\frac{\log(2/\delta')}{2N_{h}^{k-1}(s,a)}}\wedge 1 + (H-h)\sqrt{\frac{\log(2/\delta')}{2N_{h}^{k-1}(s,a)}}\wedge H\\ &\leq H\sqrt{\frac{\log(2/\delta')}{2N_{h}^{k-1}(s,a)}}\wedge H
\end{align*}
holds with probability at least $1-2\delta'$.  Taking a union bound over all states,actions,stage, and episodes by setting $\delta' = \delta/(2SAHK)$ shows that, with probability $1-\delta$, the following holds for all $s,a,h,k$,
\begin{align*}
    |r_h(s,a) + p_h(s,a)^T V_{h+1}^\star - \wh{r}_h^k(s,a) - \wh{p}_h^k(s,a)^T V_{h+1}^\star| \leq c_h^k(s,a).
\end{align*}
From here one can easily prove the main claim by induction on $h$.
\end{proof}

To simplify the notation, let $L := \frac{1}{2}\log(4SAHK/\delta)$, so that  $c_h^k(s,a) = H\sqrt{L/N_h^{k-1}(s,a)} \wedge H$. Following the idea behind the proof of the logarithmic regret bound for UCRL \citep{jaksch2010near}, we now derive an upper bound to the regret suffered by the algorithm in those episodes where the immediate regret is at least some $\epsilon > 0$. 
\begin{lemma}\label{lemma:ucbvi-regret-bad-episodes}
Let $\mathcal{K}_\epsilon := \{k\in[K] : V_1^\star - V_1^{\pi_k} \geq \epsilon\}$ be the set of (random) episode indexes where the regret is at least $\epsilon$ and $K_\epsilon := \sum_{k=1}^K \indi{k\in\mathcal{K}_\epsilon}$ be the count of such episodes. Then, with probability at least $1-2\delta$,
\begin{align*}
    \sum_{k\in\mathcal{K}_\epsilon} V_1^\star - V_1^{\pi_k} \leq H^2 \sqrt{2LSA K_\epsilon} + SAH^2 + H^2\sqrt{\frac{K_\epsilon}{2}\log(2K/\delta)}.
\end{align*}
\end{lemma}
\begin{proof}
Using the standard regret decomposition, by the optimism followed by unrolling the definition of optimistic value functions \citep[see, e.g.,][]{Dann2019certificates}, for each $k\in[K]$,
\begin{align*}
     V_1^\star - V_1^{\pi_k} \leq \wb{V}_1^k - V_1^{\pi_k} \leq \mathbb{E}^{\pi_k}\left[\sum_{h=1}^H c_h^k(s_h,a_h)\right],
\end{align*}
where the expectation is over a trajectory $\{s_h,a_h\}_{h\in[H]}$ generated by policy $\pi_k$. Therefore,
\begin{align*}
    \sum_{k\in\mathcal{K}_\epsilon} V_1^\star - V_1^{\pi_k} \leq \underbrace{\sum_{k\in\mathcal{K}_\epsilon}  \sum_{h=1}^H c_h^k(s_h^k,a_h^k)}_{(a)} + \underbrace{\sum_{k\in\mathcal{K}_\epsilon} \left( \mathbb{E}^{\pi_k}\left[\sum_{h=1}^H c_h^k(s_h,a_h)\right] - \sum_{h=1}^H c_h^k(s_h^k,a_h^k) \right)}_{(b)}.
\end{align*}
Let $\wb{N}_h^{K}(s,a) := \sum_{k\in\mathcal{K}_\epsilon}\indi{s_h^k=s,a_h^k=a}$. For term (a), we use the standard application of the pigeon-hole principle,
\begin{align*}
    (a) &= H \sum_{k\in\mathcal{K}_\epsilon}  \sum_{h=1}^H \sqrt{\frac{L}{N_h^{k-1}(s_h^k,a_h^k)}} \wedge 1 = H\sum_{h=1}^H\sum_{s,a}\sum_{k\in\mathcal{K}_\epsilon} \indi{s_h^k=s,a_h^k=a} \sqrt{\frac{L}{N_h^{k-1}(s,a)}} \wedge 1
    \\ &\leq H \sum_{h=1}^H\sum_{s,a} \left( \sum_{i=1}^{\wb{N}_h^{K}(s,a)} \sqrt{\frac{L}{i}} + 1\right) \leq H \sum_{h=1}^H\sum_{s,a} \sqrt{2L\wb{N}_h^{K}(s,a)} + SAH^2
    \\ &\leq H \sqrt{SAH\sum_{h=1}^H\sum_{s,a} 2L\wb{N}_h^{K}(s,a)} + SAH^2 = H^2 \sqrt{2LSA K_\epsilon} + SAH^2.
\end{align*}
For term (b), we need a slightly more refined martingale argument than the standard analyses. Let $G_k := \sum_{h=1}^H c_h^k(s_h^k,a_h^k)$ and note that, by definition, $G_k$ is $\mathcal{F}_{k-1}$-measurable. This is not sufficient to make term (b) a martingale since the summation is over random episode indexes. However, note that the event $\{k\in\mathcal{K}_\epsilon\}$ is $\mathcal{F}_{k-1}$-measurable too. Let $k_i$, for $i\geq 1$, denote the index of the episode where such event occurs for the $i$-th time (which is a $\mathcal{F}_{k_i-1}$-measurable random variable). We have,
\begin{align*}
    (b) = \sum_{i=1}^{K_\epsilon} \mathbb{E}[G_{k_i} | \mathcal{F}_{k_i-1}] - G_{k_i}
\end{align*}
The terms inside the sum form a martingale difference sequence with differences bounded by $H^2$ in absolute value. We only need to take a union bound over the possible values of $K_\epsilon$. Therefore,
\begin{align*}
    &\prob{\sum_{i=1}^{K_\epsilon} \mathbb{E}[G_{k_i} | \mathcal{F}_{k_i-1}] - G_{k_i} > H^2\sqrt{\frac{K_\epsilon}{2}\log(2/\delta')}} 
    \\ & \qquad \leq \sum_{k=1}^K \prob{\sum_{i=1}^{k} \mathbb{E}[G_{k_i} | \mathcal{F}_{k_i-1}] - G_{k_i} > H^2\sqrt{\frac{k}{2}\log(2/\delta')}} \leq K\delta'.
\end{align*}
Setting $\delta' = \delta/K$, we get that, with probability at least $1-\delta$,
\begin{align*}
    (b) \leq H^2\sqrt{\frac{K_\epsilon}{2}\log(2K/\delta)}.
\end{align*}
Combining this with the bound on (a), we obtain the statement, which holds with probability $1-2\delta$.
\end{proof}

\begin{theorem}\label{th:ucbvi-prob-regret}
With probability at least $1-2\delta,$ the regret of the UCBVI algorithm with Chernoff-Hoeffding bonus as defined in Lemma~\ref{lemma:ucbvi-optimism} can be upper bounded by
\begin{align*}
    \mathrm{Regret}(K) \leq \frac{4H^4 SA}{\Gamma_{\min}}\log \frac{4SAHK}{\delta} + \frac{2H^4 S^{\frac{3}{2}}A^{\frac{3}{2}}}{\Gamma_{\min}}\sqrt{\log \frac{4SAHK}{\delta}} + SAH^2.
\end{align*}
\end{theorem}
\begin{proof}
Note that, by definition of minimum policy gap, the regret at each episode is either zero or at least $\Gamma_{\min}$. Thus,
\begin{align*}
    \mathrm{Regret}(K) = \sum_{k=1}^K (V_1^\star - V_1^{\pi_k}) = \sum_{k\in\mathcal{K}_{\Gamma_{\min}}} (V_1^\star - V_1^{\pi_k}) \geq \Gamma_{\min}K_{\Gamma_{\min}}.
\end{align*}
Moreover, by Lemma~\ref{lemma:ucbvi-regret-bad-episodes}, with probability at least $1-2\delta$,
\begin{align}
    \notag\sum_{k\in\mathcal{K}_{\Gamma_{\min}}} (V_1^\star - V_1^{\pi_k}) &\leq H^2 \sqrt{2LSA K_{\Gamma_{\min}}} + SAH^2 + H^2\sqrt{\frac{K_{\Gamma_{\min}}}{2}\log(2K/\delta)} 
    \\ &\leq 2H^2 \sqrt{2LSA K_{\Gamma_{\min}}} + SAH^2. \label{eq:ucbvi-simplified-regret}
\end{align}
Therefore, combining the last two inequalities,
\begin{align*}
    \Gamma_{\min}\sqrt{K_{\Gamma_{\min}}} \leq 2H^2 \sqrt{2LSA} + SAH^2 \implies K_{\Gamma_{\min}} \leq \frac{( 2H^2 \sqrt{2LSA} + SAH^2)^2}{\Gamma_{\min}^2}.
\end{align*}
Plugging this into \eqref{eq:ucbvi-simplified-regret} and rearranging,
\begin{align*}
    \mathrm{Regret}(K) = \sum_{k\in\mathcal{K}_{\Gamma_{\min}}} (V_1^\star - V_1^{\pi_k})
    \leq \frac{8H^4 LSA}{\Gamma_{\min}} + \frac{2H^4 SA\sqrt{2LSA}}{\Gamma_{\min}} + SAH^2,
\end{align*}
which concludes the proof.
\end{proof}

Finally, Theorem \ref{lem:ucbvi.upper} is proved in the following corollary of the previous high-probability result.
\begin{corollary}\label{cor:ucbvi-exp-regret}
The expected regret of the UCBVI algorithm with Chernoff-Hoeffding bonus as defined in Lemma~\ref{lemma:ucbvi-optimism} and $\delta = 1/K$ can be upper bounded by
\begin{align*}
    \mathbb{E}[\mathrm{Regret}(K)] \leq \frac{4H^4 SA}{\Gamma_{\min}}\log (4SAHK^2) + \frac{2H^4 S^{\frac{3}{2}}A^{\frac{3}{2}}}{\Gamma_{\min}}\sqrt{\log (4SAHK^2)} + SAH^2 + 2H.
\end{align*}
\end{corollary}
\begin{proof}
Simply bound the expected regret by Theorem \ref{th:ucbvi-prob-regret} when the ``good event'' used in that theorem holds and by $2\delta KH$ when it does not. 
\end{proof}
\section{Auxiliary Results}\label{app:auxiliary}

The following lemma is the MDP variant of a well-known result from \citep{kaufmann2016complexity,garivier2019explore}. See also Lemma 5 of \cite{domingues2021episodic} and Lemma H.1 of \cite{simchowitz2019non}.
\begin{lemma}\label{lemma:change-of-measure}[Change-of-measure inequality]
Let $\M,\M'$ be two MDPs with transition probabilities $p_h,p_h'$ and reward distributions $q_h,q_h'$, respectively. For $K \in \mathbb{N}_{>0}$, let $\mathcal{F}_K$ denote the filtration generated by all trajectories up to (and including) the $K$-th episode. Then, for any $\mathcal{F}_K$-measurable event $\mathcal{E}$,
\begin{align*}
    \sum_{s\in\S}\sum_{a\in\A}\sum_{h\in[H]} \E_\M^{\mathfrak{A}}[N_{K,h}(s,a)]\mathrm{KL}_{s,a,h}(\M,\M') \geq \mathrm{kl}(\mathbb{P}_{\M}^{\mathfrak{A}}\{\mathcal{E}\}, \mathbb{P}_{\M'}^{\mathfrak{A}}\{\mathcal{E}\}),
\end{align*}
where $\mathrm{KL}_{s,a,h}(\M,\M')$ is a shorthand for $\mathrm{KL}(p_h(s,a),p_h'(s,a)) + \mathrm{KL}(q_h(s,a),q_h'(s,a))$ and denotes the KL-divergence between the transition/reward distributions of $\M,\M'$, while $\mathrm{kl}(x,y)$ denotes the KL-divergence between Bernoulli distributions with parameters $x$ and $y$, respectively.
\end{lemma}

The following is a known regret decomposition in terms of the action gaps $\Delta_\M$ (see, e.g., \citet{simchowitz2019non}). 
\begin{proposition}[Gap-based regret decomposition]\label{prop:regret-decomp}
For any learning algorithm $\mathfrak{A}$ and MDP $\M$,
\begin{align*}
    \E_{\M}^{\mathfrak{A}}\left[\mathrm{Regret}_K( \M)\right] = \sum_{s\in\S}\sum_{a\in\A}\sum_{h\in[H]} \E_\M^\mathfrak{A}[N_{K,h}(s,a)]\Delta_{\M,h}(s,a).
\end{align*}
\end{proposition}
\begin{proof}
By definition (Equation \ref{eq:regret}), we have
\begin{align}\label{eq:exp-reg}
        \E_{\M}^{\mathfrak{A}}\left[\mathrm{Regret}_K( \M)\right] := \sum_{k=1}^K \E_{\M}^{\mathfrak{A}}\left[V_{\M,0}^\star - V_{\M,0}^{\pi_k} \right].
\end{align}
Fix any episode $k\in[K]$ and focus on the difference inside the expectation. For any $s\in\S$ we have
\begin{align*}
    V_{\M,1}^\star(s) - V_{\M,1}^{\pi_k}(s) & = \underbrace{V_{\M,1}^\star(s) - Q_{\M,1}^\star(s,\pi_{k,1}(s))}_{= \Delta_{\M,1}(s,\pi_{k,1}(s))} + Q_{\M,1}^\star(s,\pi_{k,1}(s)) - V_{\M,1}^{\pi_k}(s) \\ &= \Delta_{\M,1}(s,\pi_{k,1}(s)) + Q_{\M,1}^\star(s,\pi_{k,1}(s)) - Q_{\M,1}^{\pi_k}(s, \pi_{k,1}(s))
    \\ & = \Delta_{\M,1}(s,\pi_{k,1}(s)) + p_1(s,\pi_{k,1}(s))^\tr\left(V_{\M,2}^\star - V_{\M,2}^{\pi_k}\right),
\end{align*}
where in the last equality we used the Bellman equation for $Q^\star$ and $Q^{\pi_k}$. If we now plug this back into \eqref{eq:exp-reg}, taking the expectation over the first two states,
\begin{align*}
        \E_{\M}^{\mathfrak{A}}\left[\mathrm{Regret}_K( \M)\right] := \sum_{k=1}^K \E_{\M}^{\mathfrak{A}}\left[ \Delta_{\M,1}(s_{k,1},\pi_{k,1}(s_{k,1})) + V_{\M,2}^\star(s_{k,2}) - V_{\M,2}^{\pi_k}(s_{k,2}) \right].
\end{align*}
Applying recursively the same reasoning as above to $V_{\M,2}^\star(s_{k,2}) - V_{\M,2}^{\pi_k}(s_{k,2})$, it is easy to see that
\begin{align*}
        \E_{\M}^{\mathfrak{A}}\left[\mathrm{Regret}_K( \M)\right] &:= \sum_{k=1}^K\sum_{h\in[H]} \E_{\M}^{\mathfrak{A}}\left[ \Delta_{\M,h}(s_{k,h},\pi_{k,h}(s_{k,h})) \right]\\ &= \sum_{k=1}^K\sum_{h\in[H]}\sum_{s\in\S}\sum_{a\in\A} \E_{\M}^{\mathfrak{A}}\left[ \indi{s_{k,h}=s,a_{k,h}=a}\Delta_{\M,h}(s,a) \right]\\ &= \sum_{s\in\S}\sum_{a\in\A}\sum_{h\in[H]} \E_\M^\mathfrak{A}[N_{K,h}(s,a)]\Delta_{\M,h}(s,a),
\end{align*}
where the last equation follows by definition of $N_{K,h}(s,a)$.
\end{proof}

The following result shows the relationship between policy gaps and action gaps.
\begin{proposition}[Policy-gap vs Action-gap]\label{prop:policy-vs-action-gap}
For any MDP $\M$ and policy $\pi\in\Pi$,
 \begin{align*}
	\Gamma_{\M}(\pi) 
	&= \sum_{s\in\S}\sum_{a\in\A} \sum_{h\in[H]}\rho_{\M,h}^\pi(s,a)\Delta_{\M,h}(s,a).
\end{align*}
\end{proposition}
\begin{proof}
Simply use the same proof as Proposition~\ref{prop:regret-decomp} (with $\pi_k$ replaced by $\pi$) to show that
\begin{align*}
\Gamma_\M(\pi) := V_{\M,0}^\star - V_{\M,0}^\pi = \E_{\M}^{\pi}\left[\sum_{h=1}^H\Delta_{\M,h}(s_h,a_h)\right]
\end{align*}
and rewrite the expectation in terms of $\rho_h^\pi(s,a)$.
\end{proof}

\begin{lemma}[Optimal actions vs state distribution]\label{lemma:opt-act-vs-rho}
For each return-optimal policy $\pi\in\Pi^\star$,
\begin{align*}
    \forall s\in\S,h\in[H] : \rho_h^\pi(s) > 0 \implies \pi_h(s) \in \mathcal{O}_h(s), { V_h^\pi(s) = V_h^\star(s)}.
\end{align*}
\end{lemma}
\begin{proof}
Let us proceed by contradiction. Take any $\pi\in\Pi^\star$ such that the above statement does not hold. This means that, for some $s,h$ with $\rho_h^\pi(s) > 0$, $\pi_h(s) \notin \mathcal{O}_h(s)$. Take the first stage $h$ where this occurs, which means that $\pi_{h'}(s') \in \mathcal{O}_{h'}(s')$ holds for all stages $h'<h$. Let $\pi^\star$ be any policy whose actions belong to $\mathcal{O}_{h'}(s')$ for all $s',h'$ and that is equal to $\pi$ for $h'<h$. Note that the expected return can be written as
\begin{align*}
    V_0^\pi &= 
    \sum_{s'\in\S}\sum_{h'=1}^{h-1}\rho_{h'}^\pi(s')r_{h'}(s',\pi_{h'}(s')) + \sum_{s'\in\S}\sum_{h'=h}^{H}\rho_{h'}^\pi(s')r_{h'}(s',\pi_{h'}(s'))
    \\ &= \sum_{s'\in\S}\sum_{h'=1}^{h-1}\rho_{h'}^{\pi^\star}(s')r_{h'}(s',\pi^\star_{h'}(s')) + \sum_{s'\in\S}\sum_{h'=h}^{H}\rho_{h'}^\pi(s')r_{h'}(s',\pi_{h'}(s')).
\end{align*}
Since $\rho_{h'}^\pi(s') = \P^\pi\{s_{h'}=s'\} = \sum_{s''} \P^\pi\{s_{h'}=s' | s_h=s''\}\rho_{h}^\pi(s'')$, we have that
\begin{align*}
    \sum_{s'\in\S}\sum_{h'=h}^{H}\rho_{h'}^\pi(s')r_{h'}(s',\pi_{h'}(s')) 
    &= \sum_{s''} \rho_{h}^\pi(s'')  \sum_{s'\in\S}\sum_{h'=h}^{H} \P^\pi\{s_{h'}=s' | s_h=s''\} r_{h'}(s',\pi_{h'}(s')) \\ &= \sum_{s''} \rho_{h}^\pi(s'')V_h^\pi(s'')
    = \sum_{s''} \rho_{h}^{\pi^\star}(s'')V_h^\pi(s''),
\end{align*}
where the last equality follows because the state distribution at stage $h$ is influenced only by the actions chosen at earlier stages (where $\pi$ and $\pi^\star$ match). Therefore,
\begin{align*}
    V_0^\pi = 
    \sum_{s'\in\S}\sum_{h'=1}^{h-1}\rho_{h'}^{\pi^\star}(s')r_{h'}(s',\pi^\star_{h'}(s')) +  \sum_{s''} \rho_{h}^{\pi^\star}(s'')V_h^\pi(s'').
\end{align*}
Since $\rho_{h}^{\pi^\star}(s) = \rho_h^\pi(s) > 0$ but $\pi_h(s) \notin \argmax_{a\in\A} Q_h^\star(s,a)$,
\begin{align*}
     \sum_{s''} \rho_{h}^{\pi^\star}(s'')V_h^\pi(s'') < \sum_{s''} \rho_{h}^{\pi^\star}(s'')V_h^{\pi^\star}(s'').
\end{align*}
Hence, $V_0^\pi < V_0^{\pi^\star} = V_0^\star$. Hence, $\pi$ cannot be optimal and we have a contradiction, which proves the result.
\end{proof}

\begin{lemma}[Unique optimal policy vs unique optimal state distribution]\label{lemma:unique-opt-pi-vs-rho}
If $|\mathcal{O}_{h}(s)| = 1$ for all $s\in\S,h\in[H]$, then $\rho_h^{\pi_1}(s) = \rho_h^{\pi_2}(s)$ holds for all $\pi_1,\pi_2\in\Pi^\star$ and $s\in\S,h\in[H]$. The converse is not true.
\end{lemma}
\begin{proof}
By the uniqueness of the optimal actions and Lemma \ref{lemma:opt-act-vs-rho}, any two optimal policies $\pi_1,\pi_2\in\Pi^\star$ must choose the same action in states that are visited with positive probability by at least one of the two policies. This directly implies that the two policies have the same state distribution since actions chosen on states that are not visited have no influence on the resulting distribution.

To see that the converse is not true in general, simply
take any MDP where all optimal policies have the same state distribution that places probability zero of visiting some state $s$ at stage $h$. Then, add a copy some action in such $s,h$. The optimal state distribution remains unique, but there are at least two locally-optimal actions at $s,h$.
\end{proof}

\begin{lemma}[Equivalent set of alternatives]\label{lemma:equiv-weak-alt}
Let $\Lambda^{\mathrm{wa}}(\M) := \{\M'\in\mathfrak{M} \ |\ \weakopt(\M) \cap \weakopt(\M') = \emptyset \}$ be the set of weak alternatives and $\Lambda^{\mathrm{wa}}_{\O}(\M) := \{\M'\in\mathfrak{M} \ |\ \strongopt(\M) \cap \weakopt(\M') = \emptyset \}$ be the same set where only the Bellman optimal policies of $\M$ are considered. Then,
\begin{align*}
         \Lambda^{\mathrm{wa}}(\M) \cap \Lambda^{\mathrm{wc}}(\M) = \Lambda^{\mathrm{wa}}_{\O}(\M) \cap \Lambda^{\mathrm{wc}}(\M) .
    \end{align*}
\end{lemma}
\begin{proof}
Clearly, since $\Lambda^{\mathrm{wa}}_{}(\M) \subseteq \Lambda^{\mathrm{wa}}_{\O}(\M)$, we have that
\begin{align*}
    \Lambda^{\mathrm{wa}}(\M) \cap \Lambda^{\mathrm{wc}}(\M) \subseteq \Lambda^{\mathrm{wa}}_{\O}(\M) \cap \Lambda^{\mathrm{wc}}(\M).
\end{align*}
Thus, we only need to prove that the converse is also true. 

Take any $\M' \in \Lambda^{\mathrm{wa}}_{\O}(\M) \cap \Lambda^{\mathrm{wc}}(\M)$. We need to show that $\M' \in \Lambda^{\mathrm{wa}}(\M) \cap \Lambda^{\mathrm{wc}}(\M)$. If $\weakopt(\M) \cap \weakopt(\M') = \emptyset$, then $\M' \in \Lambda^{\mathrm{wa}}(\M) \cap \Lambda^{\mathrm{wc}}(\M)$ and the result is trivial. Therefore, consider the case where this does not hold and, by contradiction, suppose that $\M' \notin \Lambda^{\mathrm{wa}}(\M) \cap \Lambda^{\mathrm{wc}}(\M)$. Since $\strongopt(\M) \cap \weakopt(\M') = \emptyset$ holds by definition, this implies that there exists a policy $\pi \in \weakopt(\M)\setminus\strongopt(\M)$ (i.e., that is return-optimal but not Bellman optimal in $\M$) such that $\pi\in\weakopt(\M')$. Now take any policy $\bar{\pi} \in \strongopt(\M)$ (note that one such policy must exist). Then, we know the following:
\begin{enumerate}
    \item $V_{\M',0}^{\pi} = V_{\M',0}^\star$ since we assumed $\M' \notin \Lambda^{\mathrm{wa}}(\M) \cap \Lambda^{\mathrm{wc}}(\M)$;
    \item $V_{\M',0}^{\bar{\pi}} < V_{\M',0}^\star$ since any Bellman-optimal policy of $\M$ is sub-optimal in $\M'$ by definition;
    \item $V_{\M,0}^\star = V_{\M,0}^\pi = V_{\M,0}^{\bar{\pi}}$ by the definition of return optimality;
    \item $V_{\M',0}^\pi = V_{\M',0}^{\bar{\pi}} = V_{\M,0}^\star$ since $\M'$ is confusing.
\end{enumerate}
But this implies that $\pi$ and $\bar{\pi}$ have the same value in $\M'$, while $\pi$ is optimal and $\bar{\pi}$ is not for such MDP. This is clearly a contradiction. Hence, $\M' \in \Lambda^{\mathrm{wa}}(\M) \cap \Lambda^{\mathrm{wc}}(\M)$, which concludes the proof.
\end{proof}

\end{document}